\documentclass{article}

\usepackage{etoolbox}
\usepackage{comment}

\newtoggle{draft}
\newcommand{\draft}[1]{\iftoggle{draft}{#1}{}}
\togglefalse{draft}

\newtoggle{neurips}
\togglefalse{neurips}
\newcommand{\neurips}[1]{\iftoggle{neurips}{#1}{}}
\newcommand{\arxiv}[1]{\iftoggle{neurips}{}{#1}}

\neurips{\usepackage[nonatbib]{neurips_2021}}

\arxiv{
\usepackage[letterpaper, left=1in, right=1in, top=1in, bottom=1in]{geometry}

}

\usepackage[utf8]{inputenc} %
\usepackage[T1]{fontenc}    %
\usepackage[colorlinks,citecolor=blue!70!black,linkcolor=blue!70!black,
urlcolor=blue!70!black,breaklinks=true]{hyperref}
\usepackage{url}            %
\usepackage{booktabs}       %
\usepackage{amsfonts}       %
\usepackage{nicefrac}       %
\usepackage{microtype}      %
\usepackage[dvipsnames]{xcolor}         %

\usepackage{enumitem}
\usepackage{breakcites}

\usepackage{caption} 
\usepackage{algorithm,algorithmicx}
\usepackage{algpseudocode}      %
\usepackage{amsthm,mkolar_definitions,color,comment}

\usepackage{mathtools}
\usepackage{amsmath}
\usepackage{bbm}
\usepackage{amsfonts}
\usepackage{amssymb}
\usepackage{nicefrac}

\usepackage{comment}

\usepackage{mathrsfs}

\arxiv{
  \usepackage{natbib}
  \bibliographystyle{plainnat}
  \bibpunct{(}{)}{;}{a}{,}{,}
}
\neurips{
  \usepackage[numbers]{natbib}
  \bibliographystyle{abbrvnat}
}

\usepackage{thmtools}
\usepackage{thm-restate}
\declaretheorem[name=Theorem]{theorem}
\declaretheorem[name=Lemma]{lemma}
\declaretheorem[name=Assumption]{assumption}
\declaretheorem[name=Example,style=plain]{example}

\declaretheorem[name=Proposition]{proposition}

\usepackage{graphicx}
\usepackage{xspace}
\usepackage{setspace}

\neurips{\usepackage[scaled=.88]{helvet}}
\arxiv{\usepackage[scaled=.88]{helvet}}
\usepackage{inconsolata}

\makeatletter
  \renewenvironment{proof}[1][Proof]%
  {%
   \par\noindent{\bfseries\upshape {#1.}\ }%
  }%
  {\qed}
\makeatother

\DeclarePairedDelimiter{\abs}{\lvert}{\rvert} %
\DeclarePairedDelimiter{\brk}{[}{]}
\DeclarePairedDelimiter{\crl}{\{}{\}}
\DeclarePairedDelimiter{\prn}{(}{)}
\DeclarePairedDelimiter{\nrm}{\|}{\|}
\DeclarePairedDelimiter{\tri}{\langle}{\rangle}

\DeclareMathOperator{\En}{\mathbb{E}}

\newcommand{\wt}[1]{\widetilde{#1}}
\newcommand{\wh}[1]{\widehat{#1}}
\def\ddefloop#1{\ifx\ddefloop#1\else\ddef{#1}\expandafter\ddefloop\fi}
\def\ddef#1{\expandafter\def\csname bb#1\endcsname{\ensuremath{\mathbb{#1}}}}
\ddefloop ABCDEFGHIJKLMNOPQRSTUVWXYZ\ddefloop
\def\ddefloop#1{\ifx\ddefloop#1\else\ddef{#1}\expandafter\ddefloop\fi}
\def\ddef#1{\expandafter\def\csname b#1\endcsname{\ensuremath{\mathbf{#1}}}}
\ddefloop ABCDEFGHIJKLMNOPQRSTUVWXYZ\ddefloop
\def\ddef#1{\expandafter\def\csname c#1\endcsname{\ensuremath{\mathcal{#1}}}}
\ddefloop ABCDEFGHIJKLMNOPQRSTUVWXYZ\ddefloop
\def\ddef#1{\expandafter\def\csname scr#1\endcsname{\ensuremath{\mathscr{#1}}}}
\ddefloop ABCDEFGHIJKLMNOPQRSTUVWXYZ\ddefloop
\def\ddef#1{\expandafter\def\csname h#1\endcsname{\ensuremath{\widehat{#1}}}}
\ddefloop ABCDEFGHIJKLMNOPQRSTUVWXYZ\ddefloop
\def\ddef#1{\expandafter\def\csname hc#1\endcsname{\ensuremath{\widehat{\mathcal{#1}}}}}
\ddefloop ABCDEFGHIJKLMNOPQRSTUVWXYZ\ddefloop
\def\ddef#1{\expandafter\def\csname t#1\endcsname{\ensuremath{\widetilde{#1}}}}
\ddefloop ABCDEFGHIJKLMNOPQRSTUVWXYZ\ddefloop
\def\ddef#1{\expandafter\def\csname tc#1\endcsname{\ensuremath{\widetilde{\mathcal{#1}}}}}
\ddefloop ABCDEFGHIJKLMNOPQRSTUVWXYZ\ddefloop

\newcommand{\ls}{\ell}

\newcommand{\veps}{\varepsilon}

\newcommand{\ldef}{\vcentcolon=}
\newcommand{\rdef}{=\vcentcolon}

\let\underbar\undefined

\makeatletter
\let\save@mathaccent\mathaccent
\newcommand*\if@single[3]{%
  \setbox0\hbox{${\mathaccent"0362{#1}}^H$}%
  \setbox2\hbox{${\mathaccent"0362{\kern0pt#1}}^H$}%
  \ifdim\ht0=\ht2 #3\else #2\fi
  }
\newcommand*\rel@kern[1]{\kern#1\dimexpr\macc@kerna}
\newcommand*\widebar[1]{\@ifnextchar^{{\wide@bar{#1}{0}}}{\wide@bar{#1}{1}}}
\newcommand*\underbar[1]{\@ifnextchar_{{\under@bar{#1}{0}}}{\under@bar{#1}{1}}}
\newcommand*\wide@bar[2]{\if@single{#1}{\wide@bar@{#1}{#2}{1}}{\wide@bar@{#1}{#2}{2}}}
\newcommand*\under@bar[2]{\if@single{#1}{\under@bar@{#1}{#2}{1}}{\under@bar@{#1}{#2}{2}}}
\newcommand*\wide@bar@[3]{%
  \begingroup
  \def\mathaccent##1##2{%
    \let\mathaccent\save@mathaccent
    \if#32 \let\macc@nucleus\first@char \fi
    \setbox\z@\hbox{$\macc@style{\macc@nucleus}_{}$}%
    \setbox\tw@\hbox{$\macc@style{\macc@nucleus}{}_{}$}%
    \dimen@\wd\tw@
    \advance\dimen@-\wd\z@
    \divide\dimen@ 3
    \@tempdima\wd\tw@
    \advance\@tempdima-\scriptspace
    \divide\@tempdima 10
    \advance\dimen@-\@tempdima
    \ifdim\dimen@>\z@ \dimen@0pt\fi
    \rel@kern{0.6}\kern-\dimen@
    \if#31
      \overline{\rel@kern{-0.6}\kern\dimen@\macc@nucleus\rel@kern{0.4}\kern\dimen@}%
      \advance\dimen@0.4\dimexpr\macc@kerna
      \let\final@kern#2%
      \ifdim\dimen@<\z@ \let\final@kern1\fi
      \if\final@kern1 \kern-\dimen@\fi
    \else
      \overline{\rel@kern{-0.6}\kern\dimen@#1}%
    \fi
  }%
  \macc@depth\@ne
  \let\math@bgroup\@empty \let\math@egroup\macc@set@skewchar
  \mathsurround\z@ \frozen@everymath{\mathgroup\macc@group\relax}%
  \macc@set@skewchar\relax
  \let\mathaccentV\macc@nested@a
  \if#31
    \macc@nested@a\relax111{#1}%
  \else
    \def\gobble@till@marker##1\endmarker{}%
    \futurelet\first@char\gobble@till@marker#1\endmarker
    \ifcat\noexpand\first@char A\else
      \def\first@char{}%
    \fi
    \macc@nested@a\relax111{\first@char}%
  \fi
  \endgroup
}
\newcommand*\under@bar@[3]{%
  \begingroup
  \def\mathaccent##1##2{%
    \let\mathaccent\save@mathaccent
    \if#32 \let\macc@nucleus\first@char \fi
    \setbox\z@\hbox{$\macc@style{\macc@nucleus}_{}$}%
    \setbox\tw@\hbox{$\macc@style{\macc@nucleus}{}_{}$}%
    \dimen@\wd\tw@
    \advance\dimen@-\wd\z@
    \divide\dimen@ 3
    \@tempdima\wd\tw@
    \advance\@tempdima-\scriptspace
    \divide\@tempdima 10
    \advance\dimen@-\@tempdima
    \ifdim\dimen@>\z@ \dimen@0pt\fi
    \rel@kern{0.6}\kern-\dimen@
    \if#31
      \underline{\rel@kern{-0.6}\kern\dimen@\macc@nucleus\rel@kern{0.4}\kern\dimen@}%
      \advance\dimen@0.4\dimexpr\macc@kerna
      \let\final@kern#2%
      \ifdim\dimen@<\z@ \let\final@kern1\fi
      \if\final@kern1 \kern-\dimen@\fi
    \else
      \underline{\rel@kern{-0.6}\kern\dimen@#1}%
    \fi
  }%
  \macc@depth\@ne
  \let\math@bgroup\@empty \let\math@egroup\macc@set@skewchar
  \mathsurround\z@ \frozen@everymath{\mathgroup\macc@group\relax}%
  \macc@set@skewchar\relax
  \let\mathaccentV\macc@nested@a
  \if#31
    \macc@nested@a\relax111{#1}%
  \else
    \def\gobble@till@marker##1\endmarker{}%
    \futurelet\first@char\gobble@till@marker#1\endmarker
    \ifcat\noexpand\first@char A\else
      \def\first@char{}%
    \fi
    \macc@nested@a\relax111{\first@char}%
  \fi
  \endgroup
}
\makeatother

\let\oldparagraph\paragraph
\renewcommand{\paragraph}[1]{\oldparagraph{#1.}}

\usepackage{thmtools}
\usepackage{thm-restate}
\makeatletter
  \renewenvironment{proof}[1][Proof]%
  {%
   \par\noindent{\bfseries\upshape {#1.}\ }%
  }%
  {\qed\newline}
\makeatother

\newcommand{\filt}{\mathfrak{F}}

\newcommand{\ErrDelBar}{\overline{\mathrm{\mathbf{Err}}}_{\Delta}(T)}

\newcommand{\pred}{\yhat}
\newcommand{\resp}{y}

\newcommand{\bigoh}{\cO}
\newcommand{\bigoht}{\wt{\cO}}

\newcommand{\expfour}{\textsf{Exp4}\xspace}

\newcommand{\squarecb}{\textsf{SquareCB}\xspace}
\newcommand{\adacb}{\textsf{AdaCB}\xspace}
\newcommand{\regcb}{\textsf{RegCB}\xspace}
\newcommand{\mainalg}{\textsf{FastCB}\xspace}

\newcommand{\chisquared}{$\chi^{2}$-divergence\xspace}

\newcommand{\tridis}{triangular discrimination\xspace}
\newcommand{\triangledis}{\tridis}
\newcommand{\rhs}{right-hand side\xspace}
\newcommand{\lhs}{left-hand side\xspace}
\newcommand{\regressor}{regression function\xspace}

\newcommand{\plugin}{plug-in\xspace}
\newcommand{\iid}{i.i.d.\xspace}

\newcommand{\K}{A}

\newcommand{\wb}[1]{\widebar{#1}}

\newcommand{\approxleq}{\lesssim}
\newcommand{\approxgeq}{\gtrsim}

\newcommand{\algcomment}[1]{\textcolor{blue!70!black}{\footnotesize{\texttt{\textbf{//
          #1}}}}}

\newcommand{\indic}{\one}

\renewcommand{\hat}[1]{\wh{#1}}

\newcommand{\logloss}{\ls_{\mathrm{log}}}
\newcommand{\logl}{\logloss}
\newcommand{\loglosst}{log loss\xspace}

\newcommand{\yhat}{\wh{y}}

\newcommand{\ybar}{\bar{y}}
\newcommand{\fstar}{f^{\star}}
\newcommand{\pistar}{\pi^{\star}}
\newcommand{\astar}{a^{\star}}

\newcommand{\dmid}{\mathop{\|}}
\newcommand{\Dkl}[2]{D_{\mathrm{KL}}\prn*{#1\dmid#2}}
\newcommand{\Dhel}[2]{D_{\mathrm{H}}\prn*{#1\dmid#2}}
\newcommand{\Dhelshort}{D_{\mathrm{H}}}
\newcommand{\Dchis}[2]{D_{\chi^{2}}\prn*{#1\dmid#2}}
\newcommand{\Dtri}[2]{D_{\Delta}\prn*{#1\dmid#2}}

\newcommand{\dkl}[2]{d_{\mathrm{KL}}\prn*{#1\dmid#2}}
\newcommand{\alg}{\mathrm{\mathbf{Alg}}_{\textup{\textsf{KL}}}}
\renewcommand{\algtext}{$\alg$\xspace} %
\newcommand{\yh}{\yhat}
\newcommand{\midsem}{\,;}

\newcommand{\RegLog}{\mathrm{\mathbf{Reg}}_{\mathsf{KL}}(T)}
\newcommand{\RegCB}{\mathrm{\mathbf{Reg}}_{\mathsf{CB}}(T)}
\newcommand{\RegBarLog}{\overline{\mathrm{\mathbf{Reg}}}_{\mathsf{KL}}(T)}
\newcommand{\RegBarCB}{\overline{\mathrm{\mathbf{Reg}}}_{\mathsf{CB}}(T)}

\newcommand{\Lhat}{L}
\newcommand{\Lhatb}{\wb{L}}
\newcommand{\Lstar}{L^{\star}}
\newcommand{\Lstarb}{\wb{L}^{\star}}

\newcommand{\fhatls}{\fhat_{\textup{\textsf{LS}}}}
\newcommand{\fhatkl}{\fhat_{\textup{\textsf{KL}}}}
\newcommand{\pihat}{\wh{\pi}}
\newcommand{\pihatls}{\wh{\pi}_{\textup{\textsf{LS}}}}
\newcommand{\pihatkl}{\wh{\pi}_{\textup{\textsf{KL}}}}

\newcommand{\fhat}{\wh{f}}

\newcommand{\Ber}{\mathrm{Ber}}
\newcommand{\ftil}{\tilde{f}}
\newcommand{\muhat}{\wh{\mu}}
\newcommand{\xone}{x\ind{1}}
\newcommand{\xtwo}{x\ind{2}}

\title{{\LARGE Efficient First-Order Contextual
  Bandits:}\\{\Large Prediction, Allocation, and Triangular Discrimination}}

\author{%
  Dylan J. Foster\\
  {\normalsize Microsoft Research, New England}\\
  {\small\texttt{dylanfoster@microsoft.com}}
  \neurips{\And}
  \arxiv{\and}
Akshay Krishnamurthy\\
{\normalsize Microsoft Research, NYC}\\
{\small\texttt{akshaykr@microsoft.com}}  %
}
\arxiv{\date{}}

\begin{document}

\maketitle

\begin{abstract}
 A recurring theme in statistical learning, online learning, and beyond is that faster convergence rates are possible for problems
  with low noise, often quantified by the performance of the best
  hypothesis; such results are known as \emph{first-order} or
  \emph{small-loss} guarantees. While first-order guarantees are
  relatively well understood in statistical and online learning, adapting to low noise in \emph{contextual bandits} (and more broadly, decision making) presents major algorithmic challenges. In
  a COLT 2017 open problem,~\citet{agarwal2017open} asked whether
  first-order guarantees are even possible for contextual bandits
  and---if so---whether they can be attained by efficient
  algorithms. We give a resolution to this question by providing an optimal and
  efficient reduction from contextual bandits to online regression
  with the logarithmic (or, cross-entropy) loss. Our algorithm is
  simple and practical, readily accommodates rich function classes,
  and requires no distributional assumptions beyond realizability. In
  a large-scale empirical evaluation, we find that our approach
  typically outperforms  comparable non-first-order methods.

On the technical side, we show that the logarithmic loss and an information-theoretic quantity called the
\emph{triangular discrimination} play a fundamental role in obtaining first-order
                       guarantees,
and we combine this observation with %
new refinements
to the regression oracle reduction framework of \citet{foster2020beyond}. 
The use of triangular discrimination yields novel results even for the classical statistical learning model, and we anticipate that it will find broader use.

\end{abstract}

\section{Introduction}
\label{sec:intro}

In the contextual bandit problem, a learning agent repeatedly makes decisions based on contextual
information, with the goal of learning a decision-making policy that
minimizes their total loss over time. This model captures simple
reinforcement learning tasks in which the
agent must learn to make high-quality decisions in an uncertain
environment, but does not need to engage in long-term planning or credit
assignment. Owing to the availability of high-quality engineered
reward metrics, contextual bandit algorithms are now routinely
deployed in production for online personalization systems \citep{agarwal2016making,tewari2017ads}.

Contextual bandits encompass both the general problem of statistical
learning with function approximation (specifically, cost-sensitive
classification) and the classical multi-armed bandit problem, yet present
algorithmic challenges greater than the sum of both
parts. In spite of these difficulties, extensive research effort over
the past decade has resulted in efficient, general-purpose algorithms,
as well as a sharp understanding of the optimal
worst-case sample complexity
\citep{auer2002non,beygelzimer2011contextual,agarwal2014taming,foster2020beyond,simchi2020bypassing}.
While the algorithmic and statistical foundations for contextual
bandits are beginning to take shape, we still lack an understanding
of \emph{adaptive} or \emph{data-dependent} algorithms that can
go beyond the worst case and exploit nice properties of
real-world instances for better performance. This is in stark contrast
to supervised statistical learning, where adaptivity has
substantial theory, and where standard algorithms (e.g., empirical risk
minimization) are known to automatically adapt to nice
data~\citep{bousquet2003introduction}. For contextual bandits,
adaptivity poses new challenges that seem to require algorithmic
innovation, and a major research frontier is to develop algorithmic principles for adaptivity and an understanding of the fundamental limits.

To highlight the lack of understanding for adaptive and data-dependent algorithms, a COLT 2017 open problem posed by
Agarwal, Krishnamurthy, Langford, Luo, and Schapire~\citep{agarwal2017open} asks
whether there exist contextual bandit algorithms that achieve a
certain data-dependent \emph{first-order} regret
  bound, which scales with the cumulative loss $\Lstar$ of the best policy,
rather than with the time horizon $T$. 
For multi-armed bandits, first-order regret bounds (also known as \emph{small-loss
    bounds} or \emph{fast rates}) typically scale as $\sqrt{\Lstar}$
  and imply faster convergence for ``easy'' problems, interpolating
  between the optimal $\sqrt{T}$ rate for worst-case instances and
  constant/logarithmic regret for noise-free instances \citep{allenberg2006hannan,foster2016learning}.
\citet{agarwal2017open} observed that existing techniques appear to be
inadequate to achieve this type of guarantee in contextual bandits. Beyond simply
asking whether first-order regret can be achieved, 
they also
asked whether it can be achieved
\emph{efficiently}, which is essential for real-world deployment. 
Subsequently,
Allen-Zhu, Bubeck, and Li~\citep{allen2018make} gave an
inefficient algorithm
with an optimal first-order regret guarantee, resolving the former question, but the
existence of efficient first-order algorithms remained open.
\looseness=-1
\vspace{-.5em}

\paragraph{Contributions}
We give the first optimal and efficient
contextual bandit algorithm with a first-order regret guarantee,
providing a resolution to the second open problem raised
by~\citet{agarwal2017open}. Our algorithm, \mainalg, builds on a
recent line
of research that develops efficient contextual bandit algorithms based
on the computational primitive of (online/offline) \emph{supervised regression}
\citep{krishnamurthy2017active,foster2018practical,foster2020beyond,simchi2020bypassing},
and is efficient in terms of queries to an \emph{online oracle} for
regression with the logarithmic loss. Beyond attaining first-order
regret, \mainalg inherits all of the benefits of recent algorithms
based on regression: it is simple and practical, accommodates flexible
function classes, requires no statistical assumptions beyond
realizability, and enjoys strong empirical performance. \looseness=-1
\vspace{-.5em}
\paragraph{Technical highlights}\looseness=-1
  By invoking the framework of regression oracles, our algorithm design
approach deviates sharply from prior approaches to first-order regret and necessitates the use of techniques that are novel even
in the context of statistical learning. At a high-level, the design of
\mainalg leverages two key techniques: \looseness=-1
\neurips{\begin{enumerate}[leftmargin=*]}
\arxiv{\begin{enumerate}}
\item \emph{First-order regret for classification via logarithmic
    loss:}
  We show that algorithms based on regression with least-squares, as
  used in prior work
  \citep{foster2020beyond,simchi2020bypassing,xu2020upper,foster2020adapting,chen2020online},
  fail to attain first-order regret, even for the simpler problem of
  cost-sensitive classification in statistical learning. In spite of
  this apparent setback, we show that regression with the
  logarithmic loss \emph{does} lead to first-order regret for statistical learning. 
  This is established through a new analysis based on an
  information-theoretic quantity
called the \emph{\tridis}
  \citep{vincze1981concept,lecam1986asymptotic,topsoe2000some}.
\item \emph{Reweighted inverse gap weighting:}  Moving from
  statistical learning to contextual
  bandits, we transform predictions into distributions
  over actions using a scale-sensitive refinement to the \emph{inverse-gap
  weighting scheme} used in the \squarecb algorithm
\citep{abe1999associative,foster2020beyond}. Our new scheme is tailored to
small losses, and we show that its error is controlled by the \tridis.

\end{enumerate}
\looseness=-1
Summarizing, our approach leverages \textbf{prediction}
  via the logarithmic loss, \textbf{allocation} via reweighted
  inverse gap weighting, and \textbf{triangular discrimination} as the bridge from prediction to allocation.

  \arxiv{\paragraph{Empirical results}}\neurips{\textbf{Empirical results.~~}} In \pref{sec:experiments}, we
  evaluate \mainalg on the large-scale contextual bandit benchmark of
  \citet{bietti2018contextual} and find that it typically outperforms
  \squarecb and other non-adaptive baselines
  \citep{foster2020instance}. Interestingly, we observe that most of
  the performance improvement can be attributed to the use of the
  logarithmic loss, while the reweighted allocation scheme provides
  modest additional benefit. These findings raise a natural question
  as to whether simply moving to the logarithmic loss can yield
  performance improvements in production contextual bandit
  deployments.

  \paragraph{On the regression oracle model} 
\neurips{\looseness=-1}
As a disclaimer, we caution
that our algorithm is efficient in terms of an oracle for online
regression, while \citet{agarwal2017open} originally asked
for an algorithm that is efficient in
terms of a \emph{cost-sensitive classification oracle} capable of solving
the policy optimization problem $\argmin_{\pi\in\Pi}\sum_{t=1}^{T}\ls_t(\pi(x_t))$.
Hence, while \mainalg is the first algorithm with
first-order regret that is efficient in \emph{any} oracle model, it
does not formally solve the original open problem. Nonetheless, %
there are strong reasons to prefer a solution based on regression over one based on classification. 
First, %
cost-sensitive classification is intractable
to implement even for simple function classes for which regression can
be solved efficiently~\citep{foster2020beyond}. 
Setting this issue aside,
(online) regression-based algorithms are typically simpler and faster than
classification-based algorithms, and multiple empirical evaluations have shown that algorithms based on
regression dominate those based on classification
\citep{foster2018practical,bietti2018contextual,foster2020instance}. 
\arxiv{Taken
together, these facts suggest that
the regression oracle framework may be the
right model to develop practical contextual bandit algorithms going
forward.}

  \arxiv{\subsection{Organization}}
  \neurips{\textbf{Organization.}~~}
  \neurips{\looseness=-1}
  \pref{sec:cb} contains our algorithm and main
  theorem. \pref{sec:analysis} describes the motivation and analysis
  ideas behind \mainalg, beginning from new techniques for statistical
  learning with regression-based classifiers. Examples for the
  main theorem are given in \pref{sec:cb_examples}, and experimental results are given in
  \pref{sec:experiments}. Detailed discussion of related work is
  deferred to \pref{sec:related}.

\neurips{\section{Main Result: An Efficient First-Order Algorithm for
    Contextual Bandits}}
\arxiv{\section{An Efficient First-Order Algorithm for Contextual Bandits}}
\label{sec:cb}

\begin{algorithm}[t]
  \setstretch{1.1}
  \begin{algorithmic}[1]
    \State \textbf{parameters}:
    \Statex{}\neurips{~~~~~~}\arxiv{~~~~\hspace{1pt}}Learning rate $\gamma>0$.
    \Statex{}\neurips{~~~~~~}\arxiv{~~~~\hspace{1pt}}Online regression oracle $\alg$.
    \For{$t=1,\ldots,T$}
    \State{}Receive context $x_t$.
    \Statex{}\neurips{~~~~~~}\arxiv{~~~~\hspace{1.3pt}}\algcomment{Compute oracle's predictions
      (Eq. \pref{eq:oracle_pred}).}
    \State{}\label{line:oracle}For each action $a\in\cA$, compute
    $\yhat_t(x_t,a)\ldef{}\alg\ind{t}(x_t,a\midsem \crl*{(x_i,a_i,\ls_i(a_i))}_{i=1}^{t-1})$.
    \State{}Let $b_t\in\argmin_{a\in\cA}\yh_{t,a}$.
    \Statex{}\neurips{~~~~~~}\arxiv{~~~~\hspace{1.3pt}}\algcomment{Reweighted inverse gap weighting.}
\State{}\label{line:igw}For each $a\neq{}b_t$, define $p_{t,a} = \frac{\pred_t(x_t,b_t)}{\K{}\pred_t(x_t,b_t) +
      \gamma(\yh_t(x_t,a) -\yh_t(x_t,b_t))}$. Let $p_{t,b_t} = 1-\sum_{a\neq{}b_t}p_{t,a}$.
    \State{}Sample $a_{t}\sim{}p_t$ and observe loss $\ls_t(a_t)$.
    \State{}Update $\alg$ with example $(x_t,a_t,\ls_t(a_t))$.
    \EndFor
  \end{algorithmic}
  \caption{\mainalg (``Fast Rates for Contextual Bandits'')}
  \label{alg:online_main}
\end{algorithm}

We begin by formally introducing the contextual bandit model. At each
round $t\in\brk*{T}$, the learner observes a context $x_t\in\cX$,
selects an action $a_t\in\cA$, then observes a loss $\ls_t(a_t)\in\brk*{0,1}$ for
the action they selected. We assume that $A := |\Acal|$ is finite and that each loss function $\ls_t:\cA\to\brk*{0,1}$ is
drawn independently from a fixed distribution
$\bbP_{\ls_t}(\cdot\mid{}x_t)$, where
$\bbP_{\ls_1},\ldots,\bbP_{\ls_T}$ and $x_1,\ldots,x_T$ are selected
by a potentially adaptive adversary.

We make a standard \emph{realizability} assumption
\citep{chu2011contextual,agarwal2012contextual,foster2018practical,foster2020beyond}. Namely,
we assume that the learner has access to a class of value functions
$\cF\subset(\cX\times\cA\to\brk*{0,1})$ (e.g., neural networks,
kernels, or forests) that models the mean of the loss
distribution.
\begin{assumption}[Realizability]
  \label{ass:realizability}
There exists a \regressor $\fstar\in\cF$ such that for all $t$,
$\fstar(x,a) = \En\brk*{\ls_t(a)\mid{}x_t=x}$.
\end{assumption}
The aim of the learner is to minimize their \emph{regret} to the
optimal policy $\pistar(x)\ldef{}\argmin_{a\in\cA}\fstar(x,a)$:
\begin{equation}
  \label{eq:regret}
\RegCB\ldef{}  
\sum_{t=1}^{T}\ls_t(a_t)
-\sum_{t=1}^{T}\ls_t(\pistar(x_t)).
\end{equation}
For each $f\in\cF$, we let $\pi_f(x)\ldef\argmin_{a\in\cA}f(x,a)$ be
the induced policy. We let $\Pi\ldef\crl*{\pi_f\mid{}f\in\cF}$ be the
induced policy class.

\arxiv{\paragraph{Additional notation}}
\neurips{\textbf{Additional notation.~~}}
We adopt standard big-oh notation, and write $f=\bigoht(g)$ to denote that $f = \bigoh(g\max\crl*{1,\mathrm{polylog}(g)})$. We use $\approxleq$ only in informal statements to highlight the most salient elements of an inequality. 
We use $a\vee{}b=\max\crl{a,b}$ and $a\wedge{}b = \min\crl{a,b}$. %

\subsection{Algorithm and Main Result}

\mainalg builds on the \squarecb algorithm of \citet{foster2020beyond},
which provides an efficient, minimax-optimal reduction from contextual
bandits to online regression
with the square loss. Compared to \squarecb and other subsequent
algorithms based on online regression~\citep{foster2020adapting,chen2020online},
the first twist here is that rather than working with the square loss,
we build on the computational primitive of online
regression with the \emph{logarithmic loss}. While this point is inconsequential for
worst-case guarantees, we establish through upper and lower bounds
(\pref{sec:plugin}) that it is a fundamental distinction
where first-order guarantees are concerned.

\paragraph{Online regression oracles}
In more detail, an online regression oracle,
which we denote by $\alg$ (for ``Kullback-Leibler'') operates in the
following protocol: For each time $t$, the algorithm receives a
context-action pair $(x_t, a_t)$, produces a prediction
$\pred_t\in\brk*{0,1}$, then receives a response $\resp_t$. The algorithm's
prediction error is measured through the binary
logarithmic/cross-entropy loss (``\loglosst'')
\begin{equation}
  \label{eq:logloss}
\logloss(\pred,\resp)\ldef{} \resp\log(1/\pred) + (1-\resp)\log(1/(1-\pred)).
\end{equation}
Its goal is to ensure that the \loglosst regret to the function
class $\cF$ is minimized for all sequences.
\begin{assumption}
  \label{ass:logloss_regret}
The algorithm $\alg$ guarantees that for every (possibly adaptively
chosen) sequence $x_{1:T},a_{1:T},y_{1:T}$, the \loglosst regret is bounded by
a function $\RegLog$:
\begin{equation}
  \sum_{t=1}^{T}\logl(\yhat_t,y_t) -
    \inf_{f\in\cF}\sum_{t=1}^{T}\logl(f(x_t,a_t),y_t) \leq{}
    \RegLog.\label{eq:log_regret}
  \end{equation}
\end{assumption}
Online regression with the logarithmic loss (or, \emph{sequential
probability assignment}) is a fundamental and well-studied problem in
online learning, and there are efficient algorithms available for
many function classes of interest \citep{cover1991universal,vovk1995game,kalai2002efficient,hazan2015online,orseau2017soft,rakhlin2015sequential,foster2018logistic,luo2018efficient}; see
\pref{sec:cb_examples} for examples. While \loglosst regret is a more stringent notion
  of performance than square loss regret, it nonetheless has a
  relatively mature theory characterizing optimal rates
  \citep{shtarkov1987universal,opper99logloss,cesabianchi99logloss,bilodeau2020tight}.

\paragraph{The algorithm}
\mainalg (\pref{alg:online_main}) is a reduction that efficiently
transforms any online regression oracle satisfying \pref{ass:logloss_regret} into a contextual bandit algorithm with an
optimal first-order regret bound. At each round $t$, the algorithm
first computes the estimated loss
\begin{equation}
\yhat_t(x_t,a)\ldef{}\alg\ind{t}(x_t,a\midsem
\crl*{(x_i,a_i,\ls_i(a_i))}_{i=1}^{t-1})\label{eq:oracle_pred}
\end{equation}
predicted by the regression oracle for each action $a$ (\pref{line:oracle}); see
  \pref{app:oracle} for a more detailed formal description of the
  oracle model. Next, \mainalg uses these estimates to assign a probability of being played to each
action $a$ via a scale-sensitive refinement to the inverse gap weighting strategy used in
  \squarecb \citep{abe1999associative,foster2020beyond}, which we call \emph{reweighted
inverse gap weighting} (\pref{line:igw}). Letting
$b_t\ldef\argmin_{a\in\cA}\pred_t(x_t,a)$ be the greedy action
according to the predicted losses, we define
\begin{equation}
  \label{eq:igw}
p_{t,a} \ldef \frac{\pred_t(x_t,b_t)}{\K{}\pred_t(x_t,b_t) +
      \gamma(\yh_t(x_t,a) -\yh_t(x_t,b_t))}\quad\forall{}a\neq{}b_t,\quad\text{and}\quad{}p_{t,b_t}\ldef1-\sum_{a\neq{}b_t}p_{t,a},
  \end{equation}
where $\gamma>0$ is a learning rate parameter.
Given this distribution, \mainalg simply samples $a_t\sim{}p_t$, then
updates the oracle with the resulting tuple $(x_t, a_t,
\ls_t(a_t))$. Our main theorem shows that this leads to an optimal
first-order regret bound.%
\begin{restatable}[Main theorem]{theorem}{maintheorem}
  \label{thm:main}
Suppose \Cref{ass:realizability,ass:logloss_regret} hold. Then \pref{alg:online_main} guarantees that for all sequences with $\En\brk[\big]{\sum_{t=1}^{T}\ls_t(\pistar(x_t))}\leq{} \Lstar$,  by choosing $\gamma=\sqrt{\K{}
  \Lstar/3\RegLog}\vee{}10\K{}$,
\begin{equation}
  \label{eq:mainalg_regret}
\En\brk{\RegCB} \leq 40\sqrt{\Lstar\cdot\K{}\RegLog} + 600\K{}\RegLog. 
\end{equation}
\end{restatable}
The dominant term in this regret bound scales with
  $\sqrt{\Lstar}$ whenever the oracle \algtext attains a fast
  $\log(T)$-type regret bound. As a simple example, whenever $\cF$ is finite, we can
instantiate \algtext so that $\RegLog\leq\log\abs{\cF}$
\citep{vovk1995game}, whereby
\mainalg enjoys optimal \citep{agarwal2012contextual} first-order regret:
\[
\En\brk{\RegCB} \leq \bigoh\prn*{\sqrt{\Lstar\cdot\K{}\log\abs{\cF}} + \K{}\log\abs{\cF}}.
\]
Beyond first-order regret, \mainalg inherits all of the advantages of online
regression-based algorithms:
\neurips{\begin{itemize}[leftmargin=*]}
  \arxiv{\begin{itemize}}
\item \emph{Efficiency and simplicity.} The memory and runtime used by
  the algorithm---on top of what is required by the
  regression oracle---scales only as $\bigoh(A)$ per step;
  implementation is trivial.
\item \emph{Flexibility.} Working with regression as a primitive means
  that the algorithm easily accomodates rich, potentially nonparametric
  function classes, and we can instantiate \pref{thm:main} to get provable
  end-to-end regret guarantees for concrete classes of interest. For example, for linear models in $\bbR^{d}$ we can
  efficiently attain
  $\RegLog\leq\bigoh(d\log(T))$ \citep{cover1991universal,kalai2002efficient}, which yields a first-order
  regret bound $\RegCB \approxleq \sqrt{\Lstar\cdot{}Ad}$; our
  result is new even for this simple special case. Similar guarantees are available for kernels, generalized linear
  models, and many other nonparametric classes. On the other hand, even for function classes where provable
algorithms are not available, regression is amenable to practical
heuristics (e.g., gradient descent for non-convex models). See \pref{sec:cb_examples} for detailed examples.
\end{itemize}
While we assume that an upper bound on the optimal loss is known for
  simplicity, one can extend to the unknown case by running the algorithm in epochs, setting $\gamma$ in
  terms of the algorithm's estimated loss
  $\Lhat_t=\sum_{\tau=1}^{t}\ls_\tau(a_\tau)$, and applying the
  doubling trick. \pref{thm:main} also readily extends to high probability.

\section{Overview of Analysis}
\label{sec:analysis}

We now outline the algorithmic principles and analysis ideas
behind \mainalg. First, in \pref{sec:plugin}, we take a step back and
consider the sub-problem of cost-sensitive classification in
statistical learning. We establish that approaches based on
least-squares fail to attain first-order regret
(\pref{thm:least_squares_plugin}) for cost-sensitive classification,
then show how to fix this problem using \loglosst regression
(\pref{thm:plugin_main}); this analysis serves as an introduction to the triangular discrimination. With this result in hand, we move to the contextual bandit setting and transform predictions into distributions
  over actions using the reweighted inverse-gap
  weighting scheme in \pref{eq:igw}, which exploits small losses. Our main result here shows that this scheme satisfies a first-order
  variant of the \emph{per-round minimax inequality} of
  \citet{foster2020beyond}, which links the instantaneous contextual
  bandit regret to the triangular discrimination for the regression oracle
  on a per-round basis (\pref{thm:per_round}). Full proofs are deferred to \pref{app:plugin,app:cb}.

\subsection{Warmup: First-order Regret Bounds for Plug-In Classifiers}
\label{sec:plugin}

For the simpler problem of cost-sensitive classification in statistical learning, the
literature on \emph{plug-in classification} shows that whenever realizability
  conditions such as \pref{ass:realizability} hold, we can obtain
  optimal worst-case regret by taking the greedy policy/classifier induced by a least-squares estimator.
  We first show that this approach fails to attain first-order regret.%
The statistical learning setting we consider is as follows.  We receive a
dataset $D_n$ consisting of $n$
context-loss pairs $(x_t, \ls_t)\sim\cD$ \iid, where
the entire loss function $\ls_t:\cA\to\brk*{0,1}$ is observed. 
Analogously to~\pref{ass:realizability}, we assume access to a function class $\cF\subseteq(\cX\times\cA\to\brk*{0,1})$ such
that
$\En_{\cD}\brk*{\ls(a)\mid{}x} = \fstar(x,a)$
for some $\fstar\in\cF$, and take $\Pi:=\crl*{\pi_f\mid{}f\in\cF}$ as
the induced class of policies. Our goal is to learn a policy
$\pihat:\cX\to\cA$ such that the regret (or, excess risk)
\begin{equation}
  \label{eq:risk_csc}
  L(\pihat) - L^{\star}
\end{equation}
is small, where $L(\pi)\ldef\En_{\cD}\brk{\ls(\pi(x))}$ and
$L^{\star}\ldef{}L(\pistar)$, with $\pistar:=\pi_{\fstar}$. Formally, this an easier problem than contextual bandits, since any
algorithm with a regret bound for contextual bandits yields a bound on the
cost-sensitive classification regret \pref{eq:risk_csc} via
online-to-batch conversion.

A classical result in statistical learning \citep{vapnik1971uniform,panchenko2002some,srebro2010smoothness} shows that
if we compute the policy/classifier
$\pihat:=\argmin_{\pi\in\Pi}\sum_{t=1}^{n}\ls_t(\pi(x_t))$ that minimizes the
empirical risk, we obtain a first-order regret
bound of the form\footnote{Following the convention in contextual
  bandit literature, we focus on finite classes with
  $\abs*{\cF}<\infty$ in this discussion,
  but one can extend our observations to general classes, e.g., using the
  machinery of \citet{zhang2006from}.}
\begin{equation}
\En\brk*{L(\pihat)}-\Lstar \approxleq \sqrt{\frac{\Lstar\cdot\log\abs{\cF}}{n}} +
\frac{\log\abs{\cF}}{n}.\label{eq:fast_rate_erm}
\end{equation}
This is an optimal first-order guarantee, but computing $\pihat$ is
typically computationally intractable, even for relatively simple
policy classes. As an alternative, the approach of plug-in
classification aims to use the realizability assumption to develop
algorithms based on the more tractable primitive of regression.
Here, another
classical result (e.g., \citet{audibert2007fast}\footnote{This result is well-known in the
  binary setting. We are not aware of a reference for the
  multiclass/cost-sensitive version here, though it is implicit in many
  recent works on contextual bandits.}), shows that if we perform
least-squares via
\[
\fhatls\ldef{}\argmin_{f\in\cF}\sum_{t=1}^{n}\sum_{a\in\cA}(f(x_t,a)-\ls_t(a))^{2},
\]
and take $\pihatls\ldef\pi_{\fhatls}$ as our classifier, then under the realizability assumption
we are guaranteed
\begin{equation}
  \label{eq:slow_rate_plugin}
\En\brk*{L(\pihatls)} - \Lstar \approxleq \sqrt{\frac{A\log\abs*{\cF}}{n}}.
\end{equation}
While this result is rate-optimal, it is not first-order, and first-order regret bounds for plug-in classification are
conspicuously absent from the literature. We show that
this is fundamental.
\begin{restatable}[Failure of least-squares for plug-in classification]{theorem}{leastsquaresplugin}
  \label{thm:least_squares_plugin}
  Let $\cA=\crl{1,2}$ and $\cX=\crl{1,2}$. For every $n>10^8$, there exists a function class $\cF\subseteq(\cX\times\cA\to\brk*{0,1})$ with $\abs{\cF}=2$, and a
realizable distribution $\cD$ such that $\Lstar\leq \tfrac{2^{7}}{n}<1$, yet 
$L(\pihatls) - \Lstar \geq{}2^{-5}\sqrt{\frac{1}{n}}$ 
with probability at least $1/10$.
\end{restatable}
Since the instance in this theorem has
$\sqrt{\frac{\Lstar\cdot{}A\log\abs{\cF}}{n}}\approxleq\frac{1}{n}$,
we conclude
that plug-in classification with least-squares fails to attain the
first-order regret bound in \pref{eq:fast_rate_erm} with constant probability; a lower
bound in expectation follows immediately.

\arxiv{ 
The main insight behind the result is that standard least-squares
incurs a poor dependence on the noise variance in the presence of
heteroscedastic noise.  Consider a (simplified) heteroscedastic regression problem in which we receive samples $\{(x_i,y_i)\}_{i=1}^n$ \iid satisfying $y_i = \fstar(x_i) + \veps_i$ and $\En\brk*{\veps_i \mid x_i} = 0$, but where $\sigma_x^2 := \En\brk*{\veps_i^2 \mid x_i = x}$ may vary substantially with $x$. 
In such settings, least-squares incurs a dependence on the worst case variance $\sup_x \sigma_x^2$, rather than the more favorable average variance $\En_x\brk*{\sigma_x^2}$.
However, the average variance plays a central role in obtaining
first-order bounds for cost-sensitive classification, since for losses
in $[0,1]$, we always have $\En_x\brk*{\Var\brk{\ell(\pistar(x)) \mid x}}
\leq \Lstar$. Thus, to prove the lower bound, we construct an instance
where the worst-case variance is constant, yet $\Lstar \approxleq
1/n$, and we show that the error for least-squares indeed
scales with the former quantity.
  }

\subsubsection{Fast Rates for Plug-In Classifiers: Triangular Discrimination and
  Logarithmic Loss}
It would appear we are at an impasse, as \pref{thm:least_squares_plugin} shows that square loss regression oracles of the type used in \citet{foster2020beyond} are unlikely to attain first-order regret bounds on their own. However, the plug-in classification approach is not completely doomed.
All we need to do to fix this issue is
  change the loss function and instead perform regression with the
  \emph{logarithmic loss}.

  To understand why plug-in least-squares fails and how it can be
  improved, it will be
  helpful to review the key steps in the analysis leading to the rate
  \pref{eq:slow_rate_plugin}.
  \newcommand{\stepref}[1]{\hyperlink{step#1}{\small\texttt{\textbf{Step #1}}}}
  \newcommand{\stepnum}[1]{\hypertarget{step#1}{{\color{blue!70!black}\small\texttt{\textbf{Step #1}}}}}
  \begin{enumerate}[leftmargin=34.8pt]
  \item[\stepnum{1}.] First, using a generic regret decomposition based on
    realizability, for any $f$ we have
    \begin{equation}
      \label{eq:stepone}
      L(\pi_f) -
      \Lstar\leq 2\max_{\pi\in\crl{\pi_f,\pistar}}\En_{\cD}\abs[\big]{f(x,\pi(x))-\fstar(x,\pi(x))}.
    \end{equation}
\item[\stepnum{2}.]
  Next, by Cauchy-Schwarz, for any policy $\pi$ we have
      \begin{equation}
        \label{eq:steptwo}
        \En_{\cD}\abs[\big]{f(x,\pi(x))-\fstar(x,\pi(x))}
        \leq{}\prn*{\En_{\cD}\abs[\big]{f(x,\pi(x))-\fstar(x,\pi(x))}^2}^{1/2},
      \end{equation}
      which we may further upper bound by $\prn*{\sum_{a\in\cA}\En_{\cD}\abs[\big]{f(x,a)-\fstar(x,a)}^2}^{1/2}$.
\item[\stepnum{3}.] Finally, under realizability, a standard concentration argument based on Bernstein's
  inequality implies that the least-squares estimator satisfies
      \begin{equation}
      \label{eq:stepthree}
      \En\brk*{\sum_{a\in\cA}\En_{\cD}\abs[\big]{\fhatls(x,a)-\fstar(x,a)}^2}
      \approxleq \frac{A\log\abs*{\cF}}{n}.
    \end{equation}
    Combining this bound with \stepref{2}, we conclude that
    $\En\brk*{L(\pihatls)} - \Lstar \leq \sqrt{A\log\abs{\cF}/n}$.
  \end{enumerate}
  The issue here is that even in the presence of low noise, the squared error in
  \pref{eq:stepthree} shrinks no faster than $\tfrac{1}{n}$. This
  holds even if
  $\Lstar\propto\tfrac{1}{n}$, as in the lower bound construction for
  \pref{thm:least_squares_plugin}.
Consequently, once we
  apply Cauchy-Schwarz in \stepref{2}, we lose all hope of attaining a first-order bound.

Our starting point toward improving this result is 
a refined application of Cauchy-Schwarz, by which we can replace the right hand side of~\pref{eq:steptwo} with
\begin{align}
&\prn*{\En_{\cD}\brk*{f(x,\pi(x))+\fstar(x,\pi(x))}\cdot{}\En_{\cD}\brk*{\frac{\prn[\big]{f(x,\pi(x))-\fstar(x,\pi(x))}^2}{f(x,\pi(x))+\fstar(x,\pi(x))}}}^{1/2}.   \label{eq:cauchy_refined}
\end{align}
The ratio term above is closely related to the \emph{triangular
  discrimination}, an information-theoretic
divergence measure which
we define for $p,q\in\bbR_{+}^{A}$ as\footnote{The \tridis is
  traditionally defined over the simplex $\Delta_{A}$, but for our
  application it is useful to work with the entire positive orthant.}
\begin{equation}
  \label{eq:tridis_def}
  \Dtri{p}{q} \ldef{} \sum_{a}\frac{(p_a-q_a)^{2}}{p_a+q_a}.
\end{equation}
The \tridis---also known as the symmetric \chisquared and Vincze-Le Cam
distance---is a fundamental, oft-overlooked quantity in information
theory \citep{vincze1981concept,lecam1986asymptotic,topsoe2000some}.
Since readers may be unfamiliar, we record
some basic facts.
\neurips{
\begin{proposition}[\citet{topsoe2000some}]
  The \tridis $D_{\Delta}$, over the domain $\Delta_{A}$, i) is the
  $f$-divergence given by $f(t) = \frac{(t-1)^{2}}{t+1}$, ii)
  is the square of a distance metric, and iii) is equivalent (up to a multiplicative constant) to both the
    Hellinger distance and Jensen-Shannon divergence.
\end{proposition}
}
\arxiv{
\begin{proposition}[\citet{topsoe2000some}]
  The \tridis $D_{\Delta}$, over the domain $\Delta_{A}$:
  \begin{enumerate}
  \item is the $f$-divergence corresponding to $f(t) = \frac{(t-1)^{2}}{t+1}$.
  \item is the square of a distance metric.
  \item is equivalent (up to a multiplicative constant) to both the
    Hellinger distance and Jensen-Shannon divergence.
  \end{enumerate}
\end{proposition}
}
The \tridis turns out to be ``just right'' for our purposes, in that it is both i) large enough to
facilitate the scale-sensitive application of Cauchy-Schwarz in
\pref{eq:cauchy_refined}, and ii) small enough (compared to the more
standard \chisquared) to facilitate minimizing from samples. 

Returning to~\pref{eq:cauchy_refined}, we can upper bound with the \tridis and leverage a certain \emph{self-bounding} property that it satisfies to arrive at the following improvement on \stepref{1}/\stepref{2}.
  \begin{restatable}[Regret decomposition for \tridis]{lemma}{triangle}
    \label{lem:triangle_selfbounding} For any
    $f:\cX\times\cA\to\brk*{0,1}$, %
        \begin{equation}
      \label{eq:triangle_selfbounding}
      L(\pi_{f}) - \Lstar
      \leq 8\prn*{\Lstar\cdot
        \En_{\cD}\brk*{\Dtri{\fstar(x,\cdot)}{f(x,\cdot)}}}^{1/2}
      + 17\En_{\cD}\brk*{\Dtri{\fstar(x,\cdot)}{f(x,\cdot)}}.
    \end{equation}
  \end{restatable}
\pref{lem:triangle_selfbounding}
  shows that low \triangledis
  (i.e. $\En_{\cD}\brk*{\Dtri{\fstar(x,\cdot)}{f(x,\cdot)}}\propto1/n$)
  suffices for an optimal first-order regret bound. What
  remains is to find an estimator $\fhat$ that minimizes this quantity
  given only samples. Our key observation here is that the \tridis
  satisfies a refined variant of Pinsker's inequality (originally due
  to \citet{topsoe2000some}), which allows us to bound it by the Kullback-Leibler divergence:%
  \begin{equation}
    \label{eq:triangle_kl}
    \Dtri{\fstar(x,\cdot)}{f(x,\cdot)}
    = \sum_{a}\frac{(f(x,a)-\fstar(x,a))^2}{f(x,a)+\fstar(x,a)}
    \leq{} 2 \sum_{a}\dkl{\fstar(x,a)}{f(x,a)},
  \end{equation}
  where $\dkl{p}{q}\ldef p\log(p/q) + (1-p)\log((1-p)/(1-q))$ is the
  binary KL-divergence. Note that the \tridis is critical here, as the \emph{opposite} inequality holds for \chisquared.
This bound suggests that we should minimize the logarithmic loss,
since---under the realizability assumption---this loss is closely
related to the KL-divergence.
In particular, we show
  (\pref{thm:plugin_chi_squared} in \pref{app:plugin}), that by taking
  the estimator
  \[
\fhatkl\ldef{}
\argmin_{f\in\cF}\sum_{t=1}^{n}\sum_{a\in\cA}\logloss(f(x_t,a), \ls_t(a)),
\]
we are guaranteed that with high probability, \neurips{$\En_{\cD}\brk[\big]{D_{\Delta}\prn{\fstar(x,\cdot)\dmid\fhatkl(x,\cdot)}} \approxleq \frac{A\log\abs*{\cF}}{n}$.}
\arxiv{
\[
  \En_{\cD}\brk*{\Dtri{\fstar(x,\cdot)}{\fhatkl(x,\cdot)}} \approxleq \frac{A\log\abs*{\cF}}{n}.
\]
}
  Putting everything together, we arrive at a first-order regret bound
for the plug-in classifier $\pihatkl\ldef{}\pi_{\fhatkl}$.\footnote{The dependence on $A$ in this result can be improved under additional
assumptions on the loss distribution. As an example, in
\pref{app:plugin} we remove the leading $A$ factor for the special case
of multiclass classification.}
\begin{restatable}[First-order regret bound for plug-in classification]{theorem}{pluginmain}
\label{thm:plugin_main}
Let $\delta \in (0,1)$. Suppose that~\pref{assum:plugin_realizability} holds. Then with probability
at least $1-\delta$, we have
\begin{align*}
L(\pihatkl) - \Lstar\leq 
16\sqrt{\frac{\Lstar \cdot \K{}\rbr{\log|\Fcal| + \log(\K{}/\delta)}}{n}} + 68\frac{\K{}\rbr{\log |\Fcal| + \log(\K{}/\delta)}}{n}.
\end{align*}
\end{restatable}

\neurips{
   Interestingly, applications of the \tridis similar to
   \pref{lem:triangle_selfbounding} have been discovered recently 
   across many branches of mathematics, including theoretical
   computer science (communication complexity)
   \citep{yehudayoff2020pointer}, probability, and group theory (construction of group homomorphisms)
   \citep{erschler2010homomorphisms,benjamini2015disorder,ozawa2015functional}. Additionally,
   \citet{bubeck2020first} use a related \emph{non-negative
     \chisquared} to provide first-order Bayesian regret bounds for
   Thompson sampling for the multi-armed bandit.\looseness=-1
   }
   \arxiv{
   Interestingly, applications of the \tridis similar to
   \pref{lem:triangle_selfbounding} have recently been discovered
   across a number of branches of mathematics, including theoretical
   computer science (communication complexity lower bounds), probability, and group theory (e.g.,
   construction of group homomorphisms)
   \citep{yehudayoff2020pointer,erschler2010homomorphisms,benjamini2015disorder,ozawa2015functional}. 
   }

\neurips{\subsection{Moving to Contextual Bandits: Inverse Gap
    Weighting meets Triangular Discrimination}}
\arxiv{\subsection{Contextual Bandits: Inverse Gap Weighting and Triangular Discrimination}}
\label{sec:cb_analysis}
\mainalg builds on the development for plug-in classifiers in
\pref{sec:plugin} but with two key differences. First, since we need
to make decisions on the fly for arbitrary sequences of contexts, the
algorithm estimates losses using an \emph{online} regression oracle
for the logarithmic loss, as described in \pref{ass:logloss_regret}. Second, and more importantly, since the
algorithm receives partial feedback, the strategy for selecting
actions is critical. Here our main technical result shows that the
reweighted
inverse gap weighting strategy \pref{eq:igw} satisfies a certain \emph{per-round}
inequality that links the instantaneous contextual bandit error to the
\tridis between the oracle's prediction $\pred_t$ and the true loss
function $\fstar$. %

\begin{restatable}[First-order per-round inequality]{theorem}{perround}
  \label{thm:per_round}
Let $y \in [0,1]^\K{}$ be given and $b \in \argmin_{a}
y_a$. Define $p_a = \frac{y_b}{\K{}y_b + \gamma (y_a - y_b)}$ for $a \ne b$, and
$p_b = 1 - \sum_{a \ne b} p_a$. If $\gamma\geq{}2\K{}$, then for all $f \in [0,1]^\K{}$ and $\astar\in\argmin_{a}f_a$, we have
\begin{align}
  \label{eq:per_round}
      \underbrace{\sum_{a}p_a(f_a-f_{\astar})}_{\textnormal{CB regret}} 
  \leq{} \underbrace{\frac{5\K{}}{\gamma}\sum_{a}p_af_a}_{\textnormal{bias
  from \emph{exploring}}} +
  \underbrace{7\gamma\sum_{a}p_a\frac{(y_a-f_a)^{2}}{y_a+f_a}}_{\textnormal{error
  from \emph{exploiting}}}.
\end{align}
\end{restatable}
The inequality \pref{eq:per_round} may be thought of as an algorithmic
analogue of the refined Cauchy-Schwarz lemma
\pref{eq:triangle_selfbounding}, with the learning rate $\gamma$
modulating the tradeoff between exploration and exploitation. Applying the inequality for each step
$t$ (with $p=p_t$, $y=\yhat_t(x_t,\cdot)$, and $f=\fstar(x_t,\cdot)$),
and using the Pinsker-type inequality~\pref{eq:triangle_kl},
we are guaranteed that
\begin{align}
  \label{eq:}
  \En\brk*{\RegCB} \leq{} \frac{5\K}{\gamma}\En\brk*{\Lhat_T} + 14\gamma\cdot\RegLog,
\end{align}
where $L_T\ldef\sum_{t=1}^{T}\ls_t(a_t)$. By a standard argument, this implies the main result in
\pref{thm:main}.

Compared to the per-round inequality used to analyze the original
version of \squarecb in \citet{foster2020beyond}, the main improvement
given by \pref{thm:per_round} is that, by reweighting---which leads to
less exploration when the optimal loss is small---we are able to
replace a constant exploration bias of order $\frac{\K{}}{\gamma}$
incurred by \squarecb with the scale-sensitive bias term
$\frac{\K}{\gamma}\cdot\sum_{a}p_af_a$ in \pref{eq:per_round}, leading
to a first-order bound. The price for this improvement is that we must now
minimize the \tridis rather than the squared error used by \squarecb, but this is taken care of by the \loglosst
oracle.\looseness=-1

\arxiv{
  \section{Examples}
  \label{sec:cb_examples}
\neurips{In this section we}\arxiv{We now} take advantage of the extensive literature on regression with the
logarithmic loss
\citep{cover1991universal,vovk1995game,kalai2002efficient,hazan2015online,orseau2017soft,rakhlin2015sequential,foster2018logistic,luo2018efficient}
and instantiate
\pref{thm:main} to give provable and efficient first-order
regret bounds for a number of function classes of interest. To the
best of our knowledge, our results are new for each of these special
cases.

\begin{example}[Finite function classes]
  If $\cF$ is a finite class, Vovk's aggregating algorithm \citep{vovk1995game} guarantees
  that\footnote{See \pref{prop:exp_concave} for a proof that the loss
    $\logl(\yhat,y)$ is mixable over the domain $\brk*{0,1}$, which is required to apply this result.}
  \begin{equation}
    \label{eq:vovk}
    \RegLog \leq \log\abs{\cF}.
  \end{equation}
  With this choice, $\mainalg$ satisfies $\En\brk{\RegCB} \leq \bigoh\prn*{\sqrt{\Lstar\cdot\K{}\log\abs{\cF}} + \K{}\log\abs{\cF}}$.
\end{example}

\begin{example}[Low-dimensional linear functions]
  Suppose that $\cF$ takes the form
  \[
\cF = \crl*{(x,a)\mapsto{} \tri*{w,\phi(x,a)}\mid{}w\in\Delta_{d}},
\]
where $\phi(x,a)\in\bbR_{+}^{d}$ is a fixed feature map with
$\nrm*{\phi(x,a)}_\infty\leq{}1$. Then the continuous exponential weights
algorithm ensures that
\[
\RegLog \leq{} \bigoh(d\log(T/d)),
\]
and can be implemented in $\mathrm{poly}(d,T)$ time per step using
log-concave sampling
\citep{cover1991universal,kalai2002efficient}. With this choice,
\mainalg satisfies
\begin{equation}
  \label{eq:linear}
  \En\brk*{\RegCB} \leq{}\bigoh\prn*{\sqrt{\Lstar\cdot\K{}d\log(T/d)} + \K{}d\log(T/d)}.
\end{equation}
\end{example}
Beyond attaining first-order regret, this bound in \pref{eq:linear} is minimax optimal
when the number of actions is constant \citep{li2019nearly}. A natural
direction for future work is to improve the result for large action spaces. Another
more practical choice for the oracle in this setting is the algorithm
of \citet{luo2018efficient}, which has slightly worse regret
$\RegLog\leq\bigoht(d^2)$, but runs in time $\bigoh(Td^{2.5})$ per step.

While first-order regret bounds for contextual bandits have primarily been
investigated for finite classes prior to this work, an advantage of
working within the regression oracle framework is that we can easily
lift our first-order guarantees to rich, nonparametric function classes.
\begin{example}[High/infinite-dimensional linear functions]
  \label{ex:highdim}
  Suppose that $\cF$ takes the form
  \[
\cF = \crl*{(x,a)\mapsto{} \tfrac{1}{2}(1+\tri*{w,\phi(x,a)})\mid{}\nrm*{w}_{2}\leq{}1},
\]
where $\nrm*{\phi(x,a)}_2\leq{}1$ is a fixed feature map. For this
setting, \citet[Section 6.1]{rakhlin2015sequential} show that the
follow-the-regularized-leader algorithm with
log-barrier regularization has\footnote{This is technically
  only proven for the case where $y\in\crl*{0,1}$, but the proof easily extends to $y\in\brk*{0,1}$.}
\[
\RegLog \leq{} \bigoh(\sqrt{T\log(T)}).
\]
This algorithm can be implemented in time $\bigoh(d)$ per step. For
this choice, \mainalg satisfies the dimension-independent rate
\begin{equation}
  \En\brk*{\RegCB} \leq{}\bigoh\prn*{(\K\Lstar)^{1/2}T^{1/4} +
    \K\sqrt{T}},
  \label{eq:example_highdim}
\end{equation}
\end{example}
Let us interpret the bound in \pref{eq:example_highdim}. First, we recall that the minimax optimal
rate for this function class is $\K^{1/2}T^{3/4}$, which the bound
above always achieves in the worst case
\citep{abe1999associative,foster2020beyond}; this
``worse-than-$\sqrt{T}$'' rate is the price we pay for working with an
expressive function class. On the other hand, if
$\Lstar$ is constant the bound in \pref{eq:example_highdim} improves
to $\bigoh(\K\sqrt{T})$, which beats the worst-case rate. While one
might hope that a tighter rate of the form, e.g.,
$(\Lstar)^{3/4}$, might be possible, by adapting
a lower bound in \citet[Section 4]{srebro2010smoothness}, one can show that the
result in \pref{eq:example_highdim} cannot be improved.

\begin{example}[Kernels]The algorithm in \pref{ex:highdim} kernelizes
  and hence can be immediately applied when
  $\cF=\crl{(x,a)\mapsto{}\frac{1}{2}(1+g(x,a))\mid{}g\in\cG}$, where
  $\cG$ is  reproducing kernel space with RKHS norm
  $\nrm*{\cdot}_{\cG}$ and kernel $\cK$. The regret bound in
  \pref{eq:example_highdim} continues to hold for this setting as
  long as
  $\nrm*{g}_{\cG}\leq{}1$ and $\cK((x,a), (x,a)) \leq{} 1$.
\end{example}
The logarithmic loss is also well-suited to generalized
linear models, as the following example highlights.%
\begin{example}[Generalized linear models]
  \label{ex:logistic}
Let
$\cF=\crl*{(x,a)\mapsto\sigma(\tri{w,\phi(x,a)})\mid{}w\in\bbR^{d},
  \nrm*{w}_2\leq{}1}$, where $\sigma(t)=1/(1+e^{-t})$ is the logistic
link function and $\phi(x,a)$ is a fixed feature map. In this case,
the map $w\mapsto{}\logloss(\sigma(\tri{w,\phi(x,a)}),y)$ is
equivalent to the
standard logistic loss function applied to $\tri{w,\phi(x,a)}$, and we can use the algorithm from
\citet{foster2018logistic} to obtain $\RegLog\leq\bigoh(d\log(T/d))$
and $\RegCB\leq\bigoht(\sqrt{\Lstar\cdot{}Ad} + Ad)$. When $d$ is
large, we can also use online gradient descent on the logistic
loss, which gives $\RegLog\leq\bigoh(\sqrt{T})$ and $\En\brk*{\RegCB} \leq{}\bigoh\prn*{(\K\Lstar)^{1/2}T^{1/4} +
    \K\sqrt{T}}$.
\end{example}

Beyond the algorithmic examples above, for general function classes \citet{bilodeau2020tight}
provide a tight characterization for the minimax optimal rates
for online regression with the logarithmic loss in terms of
\emph{sequential covering numbers} \citep{rakhlin2015sequential} for
the class $\cF$. We can use these in tandem with
\pref{thm:main} to give new regret bounds for general
classes. For example, when $\cF$ is the set of all $\brk*{0,1}$-valued
$1$-Lipschitz functions over $\brk*{0,1}^{d}$,
\citet{bilodeau2020tight} show that the optimal rate for \loglosst
regression is $\RegLog=\Theta(T^{\frac{d}{d+1}})$, which gives $\RegCB\approxleq
\bigoh\prn*{(\Lstar)^{1/2}T^{\frac{d}{2(d+1)}}+T^{\frac{d}{d+1}}}$ for
\mainalg.\neurips{\looseness=-1}

 }

\section{Experiments}
\label{sec:experiments}

\newcommand{\pvloss}{L_{\textsf{PV}}(T)}

We compared the performance of \mainalg to that of the
de-facto alternative, \squarecb \citep{foster2020beyond} in the
large-scale contextual bandit evaluation suite (``bake-off'') of
\citet{bietti2018contextual}. We found that \mainalg typically enjoys
improved performance, particularly on datasets where the optimal
loss $\Lstar$ is small. As a secondary observation, we found that
using generalized linear models with the logarithmic loss rather than a linear
model with the square loss (as in prior
work~\citep{bietti2018contextual,foster2020instance}) leads to
substantial improvements, even for \squarecb. We summarize results
here; further details
are given in \pref{app:experiments}.

\paragraph{Datasets}
The \emph{contextual bandit bake-off} is a collection of over 500
multiclass, multilabel, and cost-sensitive classification datasets
available on the \href{https://www.openml.org/}{\texttt{openml.org}}
platform \citep{vanschoren2014openml}. The collection was introduced
in \citet{bietti2018contextual} for the purpose of benchmarking
oracle-based contextual bandit algorithms.
Following~\citet{bietti2018contextual}, we use the multiclass
classification datasets from the collection (each context $x$ has a
``correct'' label $y$ associated with it) to simulate bandit feedback by assigning loss $0$ if the learner predicts the correct label and $1$ otherwise.

\paragraph{Algorithms and oracle}
We use the standard implementation of \squarecb in the
Vowpal Wabbit (VW) online learning
library,\footnote{\url{https://vowpalwabbit.org}} as used by
\citet{foster2020instance}. We also implement
\mainalg in VW.

For both algorithms, we instantiate the oracle as
  performing online logistic regression with a fixed dataset-dependent
  feature map.  This choice is convenient because i) it naturally
  produces predictions in $\brk*{0,1}$, as required by \mainalg, and
  ii), it formally meets our oracle requirements, since it is equivalent
  to online \loglosst regression with a generalized linear model.
  It can also be viewed as an
  admissible online square loss oracle, as required by \squarecb
  (see \pref{app:experiments} for further discussion). We
  additionally instantiate \squarecb with a linear model and the
  square loss, which was shown to be the strongest non-adaptive method in prior evaluations~\citep{foster2020instance}.
  We do not compare with high-performing adaptive algorithms like
  \textsf{RegCB} and \textsf{AdaCB} (as used in
  \cite{bietti2018contextual,foster2020instance}) as these algorithmic modifications
  are somewhat complementary, and we expect they can be incorporated into \mainalg.
 All oracles are trained with the default VW learning
  rule, which performs online gradient descent with adaptive
  updates~\citep{duchi2011adaptive,karampatziakis2011online,ross2013normalized}.\looseness=-1

For both \mainalg and \squarecb, we apply inverse gap weighting (the
reweighted and original version, respectively) with a
time-varying learning rate schedule in which we set $\gamma=\gamma_t$ in
\pref{line:igw} of \pref{alg:online_main} at round $t$, and likewise for
\squarecb.  Following \citet{foster2020instance}, we set $\gamma_t=\gamma_0t^{\rho}$, where $\gamma_0\in\crl{10, 50, 100, 400,
700, 10^{3}}$ and $\rho\in\crl*{.25, .5}$ are hyperparameters.

\newcommand{\fastcbl}{\mainalg.\textsf{L}\xspace}
\newcommand{\squarecbl}{\squarecb.\textsf{L}\xspace}
\newcommand{\squarecbs}{\squarecb.\textsf{S}\xspace}

\begin{figure}
\addtolength{\belowcaptionskip}{-1em}
  \centering
  ~\hfill
    \begin{tabular}{ | l | c | c | c | c | c | c | c | c | }
    \hline
      $\downarrow$ vs $\rightarrow$ & S.S & S.L & F.L \\ \hline
      SquareCB.S & - & -55 & -66 \\ \hline
      SquareCB.L & 55 & - & -11 \\ \hline
      \textbf{FastCB.L} & \textbf{66} & \textbf{11} & - \\ \hline
    \end{tabular}
    \hfill
        \begin{tabular}{ | l | c | c | c | c | c | c | c | c | }
    \hline
      $\downarrow$ vs $\rightarrow$ & S.S & S.L & F.L \\ \hline
      SquareCB.S & - & -54 & -64 \\ \hline
      SquareCB.L & 54 & - & -3 \\ \hline
          \textbf{FastCB.L} & \textbf{64} & \textbf{3} & - \\ \hline
    \end{tabular}\hfill~

\vspace{0.25cm}

  \begin{centering}
    \includegraphics[height=1.45in]{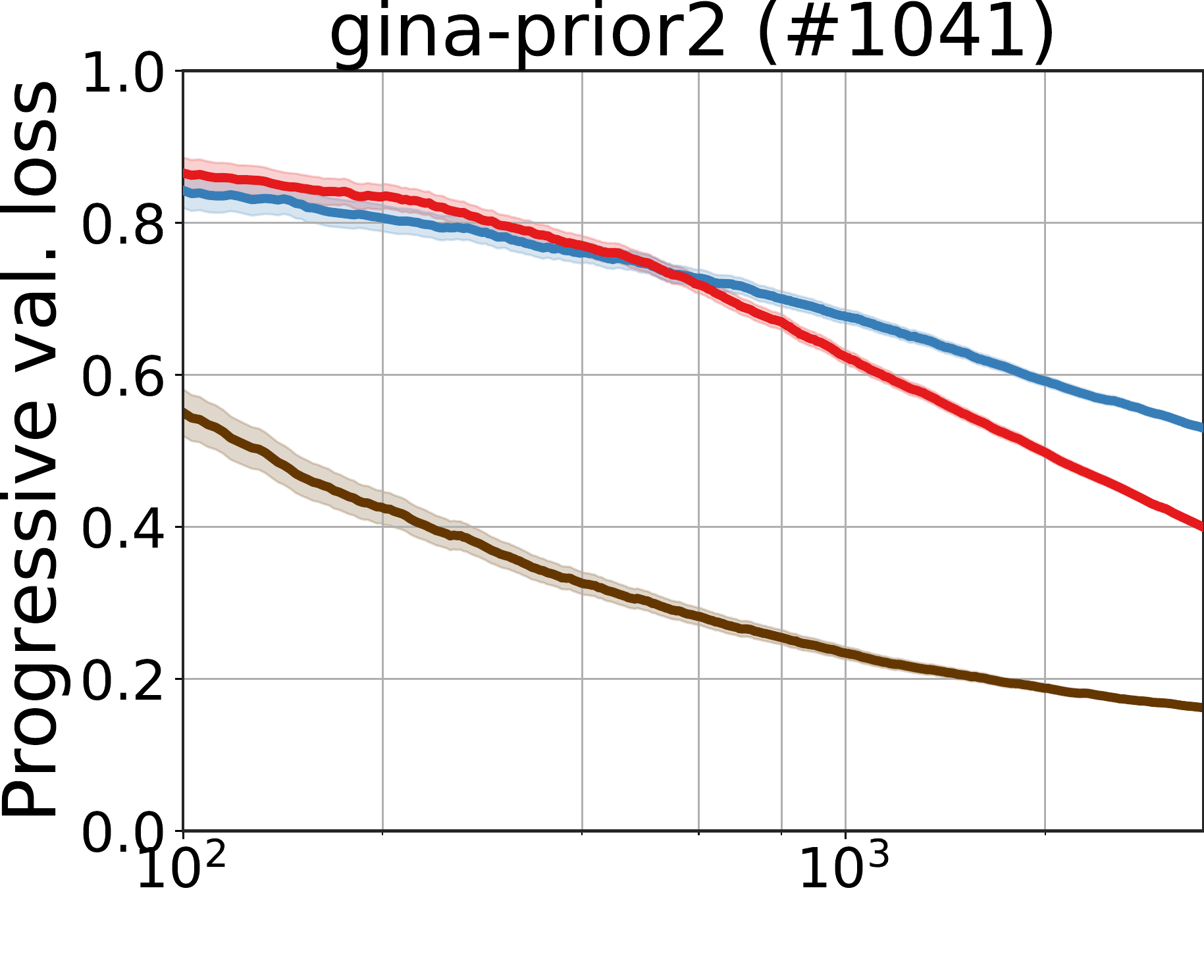}~~
    \includegraphics[height=1.45in]{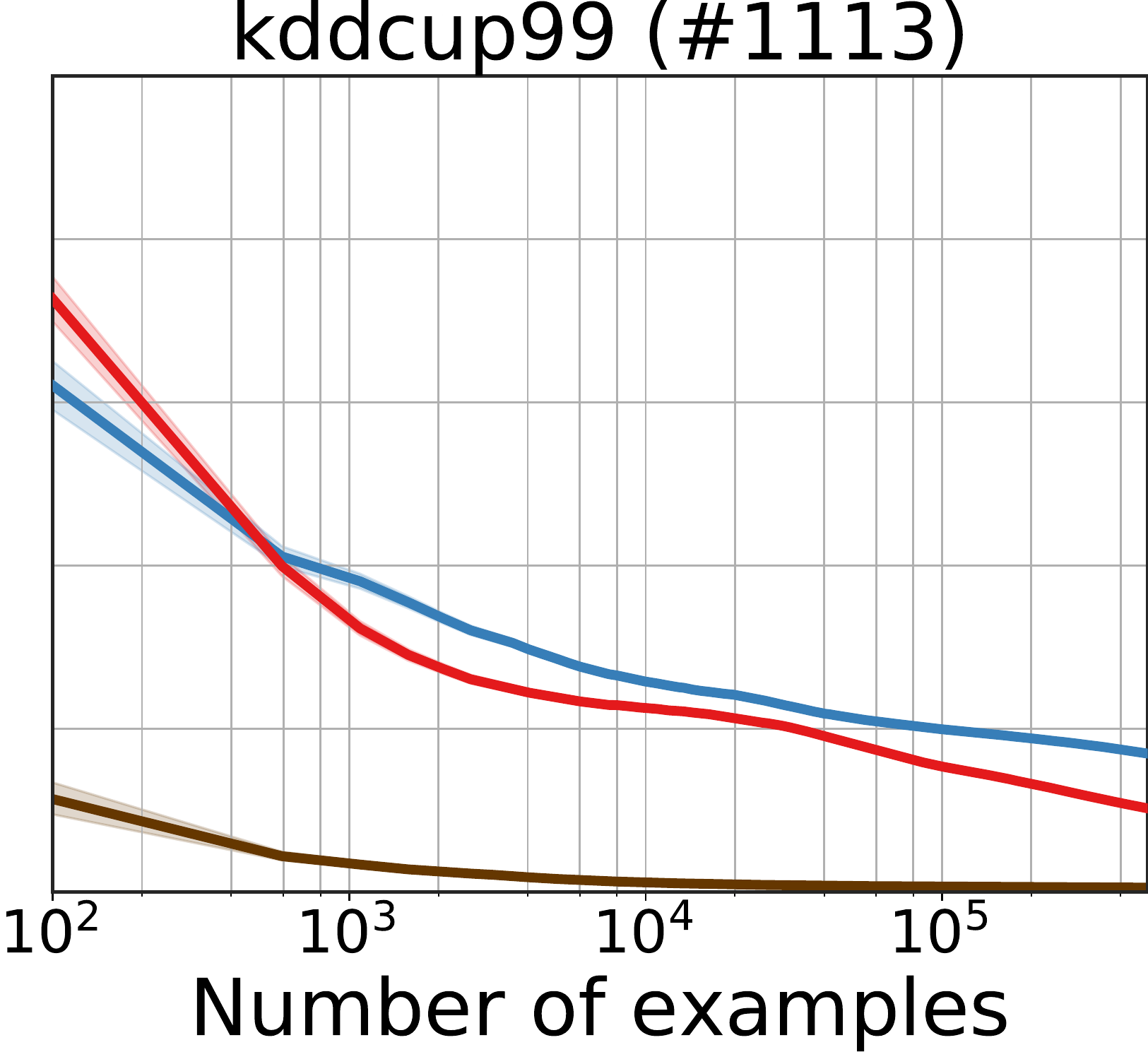}~~
    \includegraphics[height=1.45in]{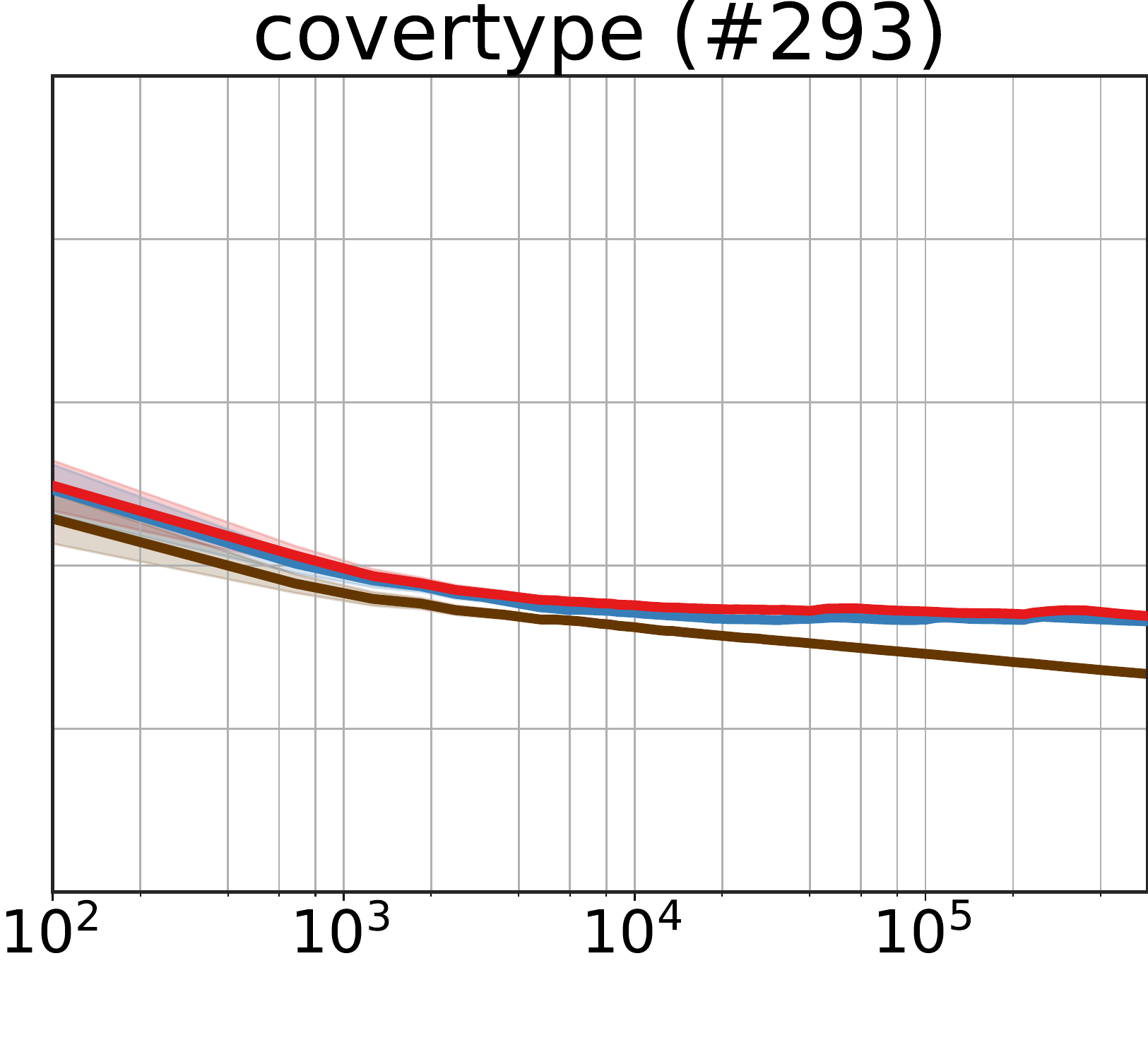}~~~~

    \medskip
    \hfill\includegraphics[width=0.45\textwidth]{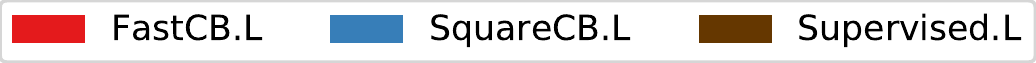}\hfill
  \end{centering}
  \caption{\emph{Top:} Head-to-head win-loss differences. Each entry indicates the statistically
    significant win-loss   difference between the row algorithm and the column
    algorithm. \emph{Top-Left:} All hyperparameters are optimized on each
    dataset. \emph{Top-Right:} Best fixed hyperparameter configuration
    across all datasets; only the oracle's learning rate is optimized per-dataset.
    \emph{Bottom:} Progressive validation results for representative
    datasets depicting significant wins for \fastcbl (left, center) and a loss (right).}
  \label{fig:experiments}
\end{figure}

\paragraph{Evaluation}
We evaluate the performance of each algorithm
using \emph{progressive validation} (PV) loss, defined as
$\pvloss = \frac{1}{T}\sum_{t=1}^{T}\ls_t(a_t)$ \citep{blum1999beating}. Following
\citet{bietti2018contextual}, we define a given algorithm as beating
another algorithm \emph{significantly} on a given dataset using an approximate $Z$-test. See
\pref{app:experiments} for details. For each pair $(a,b)$ of algorithms,
\pref{fig:experiments} (top row)
displays the number of datasets where $a$ beats
$b$ significantly, minus the number of datasets where $b$ beats $a$ significantly.

\paragraph{Results}
We
find (\pref{fig:experiments}, top row) that \mainalg with the logistic loss oracle (\fastcbl) has a positive win-loss difference against \squarecb with both logistic and square loss oracles (\squarecbl/\squarecbs),
indicating the strongest overall performance. This holds
both when hyperparameters are optimized on a per-dataset basis and
for the best global hyperparameter
configuration. \neurips{\looseness=-1}

Perhaps surprisingly, 
our results %
suggest that the largest gains
come from switching from the square loss oracle to the logistic loss
oracle (\squarecbs vs. \squarecbl), while the gains from switching from the original inverse gap
weighting strategy to our reweighted version (\squarecbl vs. \fastcbl)
are more
marginal. Inspecting the results in more detail, we find that when we compare \fastcbl and \squarecbl with hyperparameters 
optimized on a per-dataset basis, %
\fastcbl wins on 14/17 of the datasets in which either algorithm wins
significantly, and that all but two of
these 14 datasets have $\Lstar\leq{}0.2$. This suggests that the
reweighted inverse gap weighting strategy is indeed helpful when
$\Lstar$ is small. 
\pref{fig:experiments} (bottom row)
displays progressive validation performance for
\fastcbl and \squarecbl for three representative datasets which
illustrate this phenomenon.

The fact that \fastcbl does not strictly improve over \squarecbl on
every dataset, in spite of being very similar, might be attributed to
the fact that the constants in the per-round inequality
\pref{eq:per_round} are worse than those in the corresponding
inequality for \squarecbl, suggesting worse performance when
$\Lstar$ is not small. Thus, a fruitful future direction might be to
find a strategy with optimal constants for \pref{eq:per_round}. \neurips{\looseness=-1}

\arxiv{
\section{Related Work}
\label{sec:related}
First-order regret bounds have a long history throughout statistical
learning
\citep{vapnik1971uniform,panchenko2002some,srebro2010smoothness}, online
learning
\citep{freund1997decision,auer2002adaptive,PLG,cesa2007improved,luo2015achieving,koolen2015second,foster2015adaptive},
and bandits \citep{allenberg2006hannan,foster2016learning,agarwal2017open,lykouris2017small,allen2018make}. Below we highlight some of the
most relevant lines of work.

\paragraph{Statistical learning and plug-in classification}
Beginning with the work of \citet{vapnik1971uniform} for VC classes,
classical work in statistical learning
\citep{panchenko2002some,srebro2010smoothness} provides first-order
regret  (or, excess risk) bounds for \emph{empirical risk minimization} which, in our
setting, corresponds to the (typically intractable) policy
optimization problem
$\argmin_{\pi\in\Pi}\sum_{t=1}^{T}\ls_t(\pi(x_t))$.
These results are also sometimes referred to as relative deviation
bounds.

In the realizable setting (i.e., under \pref{ass:realizability}), the process of fitting a model $\fhat$ for the losses using
regression and then performing classification with the induced
classifier $\pi_{\fhat}$ is often referred to as \emph{plug-in
  classification}
\citep{yang1999minimax,audibert2007fast,devroye2013probabilistic}. While
these works establish worst-case optimal guarantees for plug-in
classifiers, first-order regret bounds are---to the best of our
knowledge---unexplored, and our observations regarding the
suboptimality of least-squares and optimality of \loglosst regression
are new. 

\paragraph{Bandits}
First-order regret bounds for multi-armed bandits appear in
\citet{allenberg2006hannan} (see also \citet{foster2016learning,bubeck2020first}), and
have been extended to the semi-bandit framework
\citep{neu2015first,lykouris2017small} and linear bandits
\citep{ito2020tight}. For contextual bandits, \citet{agarwal2017open}
show that many common algorithms fall short of achieving first-order
regret, and we are not aware of any optimal first-order algorithms outside the solution of
\citet{allen2018make}, even if one disregards efficiency or considers additional assumptions such as realizability.

On the technical side, \citet{bubeck2020first} provide first-order
regret bounds for Thompson sampling for the multi-armed bandit in the Bayesian
setting. Their approach takes advantage of a certain \emph{nonnegative \chisquared}
which is closely related to the \tridis we work with. Curiously,
their analysis uses this divergence to measure distance between
(posterior) distributions over actions, whereas we use the \tridis to
measure distance between regression functions. It would be interesting
to understand whether there are deeper (e.g., primal-dual) connections between these approaches.

\paragraph{Fast rates under margin/gap conditions}
Another line of work on plug-in classifiers aims for faster rates under
various margin assumptions, and---similar to our work---observes that least-squares can be
suboptimal in certain settings \citep{audibert2007fast}. Fast rates based on margin conditions
are distinct from first-order bounds (neither type of bound
implies the other in general), but it would be interesting to
understand their relationship more closely. Recent work \citep{foster2020instance}
extends these developments to contextual bandits and provides
logarithmic regret bounds based on similar gap/margin conditions. As in statistical learning, these types of guarantees are
incomparable to first-order regret bounds.

\paragraph{Heteroscedastic regression}
Our observations regarding suboptimality of least-squares for plug-in
classification are also closely related to regression with
heteroscedastic noise (\citet{carroll1982adapting}; \citet[][Chapter
6]{takeshi1985advanced}). Consider a regression setting where we receive
variables $\crl*{(x_i,y_i)}_{i=1}^{n}$ \iid, with
$y_i=\fstar(x_i)+\veps_i$ for some $\fstar\in\cF$, where
$\En\brk*{\veps_i\mid{}x_i}=0$, and our goal is to produce an estimator
such that the $L_1$-error $\En\abs{\fhat(x)-\fstar(x)}$ is small. In
the heteroscedastic model, the noise variance
$\sigma^{2}_x\ldef{}\En\brk{\veps_i^2\mid x_i=x}$ may vary as a
function of $x$. Using the same construction as
\pref{thm:least_squares_plugin}, one can show that standard
least-squares incurs error scaling with the worst-case variance
$\sup_x \sigma_x^2$, while, if the variances were known, weighted
least-squares with weights $w_x:=1/\sigma_x^2$ would yield error scaling
with the more favorable average variance $\EE\brk*{\sigma_x^2}$. Key to our results
is that for responses in $[0,B]$, we have $\EE\brk*{\sigma_x^2} \leq
B\cdot \EE\brk*{\fstar(x)}$ and, as we show, the logarithmic loss
achieves error scaling with the latter quantity \emph{without
  knowledge of the variances}. We mention in passing that regression with
heteroscedastic noise has found recent use in the context of
reinforcement learning with linear function
approximation \citep{zhou2020nearly,zhang2021variance}.

 }

\section{Discussion}
\label{sec:discussion}
We have given the first efficient algorithm with optimal
first-order regret for contextual bandits, resolving a variant of the open
problem posed by~\citet{agarwal2017open}. Let us 
briefly mention some
extensions. First, we believe that our techniques can also be used to
obtain first-order guarantees for stochastic
contextual bandits with an \emph{offline} \loglosst oracle~(à la
\citet{simchi2020bypassing})---albeit with a more technical analysis. As another
extension, in~\pref{app:extensions} we show how to use our method to efficiently obtain a
first-order regret bound when working with rewards rather than losses. Such a guarantee is useful when no policy accumulates
much reward, as is common in personalization applications. Several
other extensions appear to be straightforward, including working with infinite action
spaces~\citep{foster2020adapting}.%

We close with some directions for future work. Directly relevant
to our theoretical results is to continue the investigation into
adaptivity in contextual bandits and reinforcement
learning. More broadly, while triangular discrimination has been used in various
mathematics disciplines, we are not aware of many applications in
algorithm design. Are there other uses for the triangular
discrimination in machine learning?  We look forward to pursuing these
directions.

\neurips{\begin{ack}}
 \arxiv{\subsection*{Acknowledgements}}
 We thank Sivaraman Balakrishnan, John Langford, Zakaria Mhammedi, and
 Sasha Rakhlin for many helpful discussions. We also thank Sasha
 Rakhlin for providing Google Cloud credits used to run the
 experiments.
\neurips{\end{ack}}

\neurips{\newpage}

\bibliography{refs}

\newpage

\neurips{
\section*{Checklist}
\input{section_checklist}
}

\newpage

\appendix

\draft{
\section*{Notation/Conventions}
\begin{itemize}
\item Upper case section titles (i realized all my award winning papers have this lol)
\item $\K$: number of actions
\item Big-O: $\bigoh$
\item Indicator: $\indic$
\item Let's just say ``regression function'' and not ``regressor''.
\item \iid setting:
  \begin{itemize}
  \item Joint distribution $\cD$ over $(x,\ls)$.
  \item $\En_{\cD}$ and $\bbP_{\cD}$ for expectation and probability
    under this law.
    \item Suppress the $\cD$ subscript when this is clear from
      context.
    \item $\bbP\brk{\cdot}$, not $\bbP(\cdot)$
    \end{itemize}
  \item Divergences:
    \begin{itemize}
    \item $\Dkl{P}{Q}$, $\Dhel{P}{Q}$, $\Dchis{P}{Q}$, $\Dtri{P}{Q}$.
    \end{itemize}
\end{itemize}
}

\neurips{
\section{Additional Discussion}
\label{app:additional}
\input{appendix_additional}
}

\section{Proofs for Plug-In Classification Results (\pref{sec:plugin})}
\label{app:plugin}

\subsection{Proof of \pref{thm:least_squares_plugin}}
  \newcommand{\LhatLS}{\wh{L}_{\textup{\textsf{LS}}}}
  \newcommand{\aone}{a\ind{1}}
  \newcommand{\atwo}{a\ind{2}}
  \newcommand{\vepsn}{\veps_n}
  \newcommand{\mun}{\mu_n}
  \newcommand{\nun}{\nu_n}
  \newcommand{\pn}{p_n}
\leastsquaresplugin*

\begin{proof}%
  Let $\LhatLS(f) =
  \frac{1}{n}\sum_{t=1}^{n}\sum_{a\in\cA}(f(x_t,a)-\ls_t(a))^2$ be the
  empirical square loss, so that
  $\fhatls=\argmin_{f\in\cF}\LhatLS(f)$. We adopt the shorthand
  $\vepsn=1/n$ throughout the proof.
  \paragraph{Construction}
We define $\cX=\crl*{\xone,\xtwo}$ and $\cA=\crl*{\aone,\atwo}$, so that there are only two possible
contexts and actions.

The data-generating process for our construction has three parameters, $\mun$, $\nun$,
and $\pn$. We choose $\bbP_{\cD}(x=\xone)=1-\pn$, and define
$\fstar$ and the conditional loss distribution as follows:
\begin{itemize}
\item $\fstar(\xone,\aone)=\mun$ and $\fstar(\xone,\atwo)=\nun$, where
  $\mun<\nun$. We choose $\ls(\aone)\sim\Ber(\mun)\mid{}\xone$ and
  $\ls(\atwo)=\nun~\text{a.s.}\mid{}\xone$.
\item $\fstar(\xtwo,\aone)=\fstar(\xtwo,\atwo)=\tfrac{1}{2}$. We choose $\ls(\aone)\sim\Ber(\tfrac{1}{2})\mid{}\xtwo$ and
  $\ls(\atwo)=\tfrac{1}{2}~\text{a.s.}\mid{}\atwo$.
\end{itemize}
We take $\cF=\crl{\fstar,\ftil}$, where $\ftil$ will be fully
specified in the sequel, but is chosen to satisfy
$\ftil(x,\atwo)=\fstar(x,\atwo)$ for all $x$. This, combined with the
fact that $\ls(\atwo)$ is deterministic conditioned on $x$, means that
our analysis will only concern the realized outcomes for $\ls(\aone)$.

The high level idea for our construction is to set
$\pn,\mun\propto\vepsn=1/n$, which ensures that
{$\Lstar\leq(1-\pn)\mun + \pn\approxleq\frac{1}{n}$}, then show
that if we choose $\ftil(\cdot,\aone)\approx(\sqrt{\vepsn},0)$, we have $\fhatls=\ftil$
with constant probability. We then choose $\nun\approx\sqrt{\vepsn}/2$,
which implies that $\pi_{\ftil}(\xone)=\atwo\neq{}\pistar(\xone)$, and consequently
\[
L(\pihatls) = L(\pi_{\ftil}) \approxgeq{}(1-\pn)\cdot\fstar(\xone,\pi_{\ftil}(\xone))=(1-\pn)\cdot\nun
\approxgeq \sqrt{\vepsn}.
\]
We make this approach formal below.
\paragraph{Bad event}

Let $n_1$ and $n_2$ be the number of examples for which $x=\xone$ and
$x=\xtwo$. Let $n_1(0)$ and $n_1(1)$ be the number of examples for which $x=\xone$
and $\ls(\aone)=0$ or $\ls(\aone)=1$, respectively, and let $n_2(0)$ and $n_2(1)$ be
defined likewise. We restrict to $n\geq{}4$ going forward so that $\vepsn\leq1/4$.

Let $\muhat_1=\frac{1}{n_1}\sum_{i:x_i=\xone}\ls(\aone)$ (whenever $n_1>0$), and let $\muhat_2$
be defined likewise.

We prove the following proposition, which states that a certain
event that is unfavorable for the least-squares estimator occurs with
constant probability.
\begin{proposition}
  \label{prop:few_b_examples}
  Let $n\geq{}256$. Then if we set $\pn=\vepsn$ and $\mun=2^7\vepsn$, the
  following event holds with probability at least $1/10$.
  \begin{enumerate}
  \item $n_2=n_2(0)=1$, and in particular $\muhat_2=0$.\label{item:few_b_1}
  \item $n_1\geq\frac{3}{8}n$. \label{item:few_b_2}
  \item $\muhat_1 \leq{} \frac{3}{2}\mun$. \label{item:few_b_3}
  \end{enumerate}
\end{proposition}

Going forward, we adopt the parameter setting in
\pref{prop:few_b_examples} and condition on the event in the
proposition, which we denote by $\scrE$. Note that this parameter setting ensures that
\[
\Lstar =(1-\vepsn)\fstar(\xone,\aone) + \vepsn\fstar(\xtwo,\aone) =
(1-\vepsn)\mun + \frac{\vepsn}{2} \leq{} 2^{8}\vepsn,
\]
as long as $\mun<\nun$.
\paragraph{Lower bound under the bad event}

Next, we observe that for both $f\in\cF$, since $f(x,\atwo)$
perfectly predicts $\ls(\atwo)$ for all $x$, we have
\[
\LhatLS(f) \equiv \frac{n_1}{n}(f(\xone,\aone)-\muhat_1)^{2} + \frac{n_2}{n}(f(\xtwo,\aone)-\muhat_2)^{2},
\]
up to additive noise that depends only on the realization of the
dataset, not on the function $f$ under consideration. Since our
argument only depends on the relative value of $\LhatLS$, we identify
$\LhatLS$ with this representation going forward. We first observe that
conditioned by \pref{prop:few_b_examples} (\pref{item:few_b_1}), we have $\muhat_2=0$, so
that
\[
\LhatLS(\fstar) \geq{} \frac{n_2}{n}(\fstar(\xtwo,\aone)-\muhat_2)^{2} =
\vepsn\cdot{}(\fstar(\xtwo,\aone))^{2} = \frac{\vepsn}{4}.
\]
Here we use that $n_2=1$ under the bad event and that $\vepsn=1/n$.
On the other hand, if we set $\ftil(\xtwo,\aone)=0$, we have
\begin{align*}
\LhatLS(\ftil) = \frac{n_1}{n}(\ftil(\xone,\aone)-\muhat_1)^2 &\leq{} 2(\ftil(\xone,\aone))^{2} +
                                                              2\muhat_1^{2} \\
  & \leq{}2(\ftil(\xone,\aone))^{2} + 2^{3}\mun^{2} \\ &\leq{} 2(\ftil(\xone,\aone))^{2} + 2^{17}\vepsn^{2},
\end{align*}
where we have used \pref{prop:few_b_examples}
(\pref{item:few_b_3}). Note that as long as $\vepsn<2^{-20}$, we have $2^{17}\vepsn^{2}<\vepsn/8$. If
this is satisfied, then by choosing $\ftil(\xone,\aone)=\sqrt{\vepsn/16}$, we have
\[
\LhatLS(\ftil) < \frac{\vepsn}{4} \leq{} \LhatLS(\fstar),
\]
and we conclude that $\fhatls=\ftil\neq\fstar$ whenever $\scrE$
occurs. 

To conclude, we set $\nun=\sqrt{\vepsn}/8$. Since
$\ftil(\xone,\atwo)=\nun$, we have
$\ftil(\xone,\atwo)<\ftil(\xone,\aone)$, so that
$\pi_{\ftil}(\xone)=\atwo\neq\pistar(\xone)$; this choice satisfies
$\mun<\nun$ as required as long as $\vepsn<2^{-20}$. Finally, we observe
that
\[
L(\pi_{\ftil}) - \Lstar = (1-\vepsn)(\nun-\mun) \geq{}
\frac{1}{2}(\sqrt{\vepsn}/8-2^{7}\vepsn) > 2^{-5}\sqrt{\vepsn},
\]
as long as $\vepsn<2^{-22}$.
\end{proof}

\begin{proof}[\pfref{prop:few_b_examples}]
  Let $\scrE_1$, $\scrE_2$, and $\scrE_3$ denote the respective events
  in \pref{prop:few_b_examples}. We lower bound their probabilities one
  by one.
  \paragraph{Event $\scrE_1$}
We calculate
  \[
    \bbP(n_2=1)=\sum_{i=1}^{n}\vepsn\cdot(1-\vepsn)^{n-1}=\frac{1}{1-\vepsn}(1-\vepsn)^{1/\vepsn}\geq{}e^{-1},
  \]
  where we have used that $(1-1/x)^x\geq{}e^{-1}(1-1/x)$ for
  $x\geq{}1$. Hence, since $\ls(\aone)\sim\Ber(\tfrac{1}{2})$ given $x\ind{2}$, $\scrE_1$ happens with probability at least
  $1-\delta_1$ for $\delta_1\ldef1-e^{-1}/2$.
  \paragraph{Event $\scrE_2$}
  We recall a standard multiplicative variant of the Chernoff bound.
  \begin{lemma}[Chernoff bound (e.g., \citet{boucheron2013concentration})]
    \label{lem:chernoff}
    Let $Y_i\sim\Ber(\mu)$ \iid. Then for any
    $x\in\brk*{0,1/2}$,
    \[
\bbP\prn*{\sum_{i=1}^{n}Y_i\geq{}(1+x)\mu{}n}\vee \bbP\prn*{\sum_{i=1}^{n}Y_i\leq(1-x)\mu{}n}\leq{}e^{-\frac{1}{4}x^{2}\mu{}n}.
    \]
  \end{lemma}
As long as $\pn=1/n\leq{}1/4$, \pref{lem:chernoff} implies that
$n_1\geq{}\frac{3n}{8}$ with probability at least
$1-e^{-\frac{3n}{64}}\rdef{}1-\delta_2$, so that event $\scrE_2$ holds.
\paragraph{Event $\scrE_3$}

We observe that conditioned on the
realization of $x_1,\ldots,x_n$, \pref{lem:chernoff} implies that
\[
\muhat_1 \leq{} \frac{3}{2}\mun
\]
with probability at least $1-e^{-\frac{1}{16}\mun{}n_1}$. Conditioned
on $\scrE_2$, this probability is at least
$1-e^{-\frac{3}{128}\mun{}n}$. Since $\mun=128/n$, which is admissible
whenever $n\geq{}256$, we conclude that $\scrE_3$ holds with
probability at least $1-e^{-3}\rdef{}1-\delta_3$ given $\scrE_2$.

\paragraph{Wrapping up}
Taking a union bound, we have that $\scrE=\bigcup_{i=1}^{3}\scrE_i$
occurs with probability at least $1-\sum_{i=1}^{3}\delta_i \geq e^{-1}/2-e^{-12}-e^{-3}\geq{}1/10$.

\end{proof}

\subsection{Proof of \pref{thm:plugin_main}}
\newcommand{\err}{\mathrm{err}}

\subsubsection{Overview of Results}

Recall that we work in the plug-in classification setting of
\pref{sec:plugin}, where $\cX$ is the feature/context space, $\cA$ is the
label/action space, and $\cD$ is the joint distribution over
context-loss pairs $(x,\ls)$. We take a class of regression functions
$\cF\subseteq(\cX\times\cA\to\brk*{0,1})$ as a given and make the
following realizability assumption.
\begin{assumption}
\label{assum:plugin_realizability}
Define $f^\star(x,a) = \EE_{\Dcal}[\ls(a) \mid x]$. We assume
$f^\star \in \Fcal$.
\end{assumption}

Under realizability, the optimal classifier is $\pi^\star(x)\ldef
\argmin_{a \in \Acal} f^\star(x,a)$, and we have $L(\pi) = \EE\brk{
f^\star(x,\pi(x))}$. Motivated by realizability, the \plugin approach
to classification finds and estimator $\hat{f}\in\cF$ and returns the induced classifier $\hat{\pi}(x)\ldef
\argmin_{a \in \Acal} \hat{f}(x,a)$. In this section, we estimate the
losses using the following \loglosst regression problem.
\begin{align}
\label{eq:plugin_erm}
\fhatkl \gets \argmin_{f \in \Fcal} \frac{1}{n}\sum_{i=1}^n \sum_{a \in \Acal} \ls_i(a) \log(1/f(x_i,a)) + (1-\ls_i(a))\log(1/(1-f(x_i,a))).
\end{align}

For the resulting classifier $\pihatkl\ldef\pi_{\fhatkl}$, we prove the following theorem.

\pluginmain*

\paragraph{Multiclass classification}
We also provide a refinement of \pref{thm:plugin_main} for the important special case of
multiclass classification. Here, rather than observing a cost function $\ls
\in [0,1]^\K{}$ we simply observe a label $y \in \cA$ and the goal is to
predict the correct label. Formally, the distribution $\Dcal$ is
supported on $\Xcal \times \cA$ and we measure the error of a
classifier as $\err(\pi) := \PP_{\cD}\brk{\pi(x) \ne y}$. This can be seen as a
special case of cost-sensitive classification by defining loss
function $\ls(a)=\indic\crl*{a\neq{}y}$, and the realizability assumption is
as before, so that $f^\star(x,a) = \PP_{\cD}\brk{a \neq y \mid x}$. 

In this setting, rather than reducing to Bernoulli MLE, it is more
natural to reduce to multinomial MLE. Since our function class is
designed to predict the probability that a given action is
\emph{wrong} (that is, $\PP_{\cD}\brk{y = a \mid
  x } = 1 -f^\star(x,a)$), the multinomial MLE problem is
\begin{align*}
\fhatkl\gets \argmax_{f \in \Fcal} \frac{1}{n} \sum_{i=1}^n \log(1 - f(x_i,y_i)).
\end{align*}
The resulting policy is $\pihatkl := \argmin_{a} \fhatkl(x,a)$, for
which we establish the following guarantee.

\begin{theorem}
\label{thm:multinomial_plugin}
Let $\delta \in (0,1)$ and consider the multiclass classification setting
under~\pref{assum:plugin_realizability}. Then with probability at least $1-\delta$,
\begin{align*}
\err(\pihatkl) - \err(\pi^\star) \leq 8\sqrt{\frac{\err(\pi^\star)
    \cdot 2\log(|\Fcal|/\delta)}{n}} + 34\frac{\log(|\Fcal|/\delta)}{n}.
\end{align*}
\end{theorem}
Compared to \pref{thm:plugin_main}, we see that by working in the simpler
multiclass classification setting, we can remove the dependence
on $A$ from the theorem.

\subsubsection{Preliminaries}
For discrete distributions $p,q \in \Delta_\K$, the Hellinger distance is defined as \[\Dhelshort^2(p\dmid{}q)
= \frac{1}{2}\sum_{a}(\sqrt{p_a} - \sqrt{q_a})^2.\] For scalars $p,q \in
[0,1]$ we overload notation and interpret $\Dhelshort^2(p\dmid{}q) \equiv \Dhelshort^2( (p,1-p) \dmid{} (q,1-q))$ as the Hellinger divergence between the implied Bernoulli distributions.
We similarly overload $\Dtri{p}{q}\equiv\Dtri{(p,1-p)}{(q,1-q)}$ as
the Bernoulli triangular discrimination when given scalar arguments.

The following useful result relates the Hellinger distance to the \tridis.
\begin{proposition}
\label{prop:hellinger_chi_squared}
For all $p,q \in [0,1]$, we have
\begin{align*}
\Dhelshort^2(p\dmid{}q) \geq \frac{1}{4}\Dtri{p}{q} \geq \frac{1}{4}\frac{(p-q)^2}{(p+q)}.
\end{align*}
\end{proposition}

\subsubsection{Proof of \pref{thm:plugin_main} and \pref{thm:multinomial_plugin}}
We focus on proving~\pref{thm:plugin_main} and provide a sketch
for~\pref{thm:multinomial_plugin}, which is quite similar. For the
former, the core of the argument is a generalization guarantee for
$\fhatkl$.
\begin{theorem}
\label{thm:plugin_chi_squared}
Under the conditions of~\pref{thm:plugin_main},  with probability at
least $1-\delta$, we have
\begin{align}
\EE_{\Dcal} \sbr{ \sum_{a \in \Acal} \frac{(\fhatkl{}(x,a) -
  f^\star(x,a))^2}{\fhatkl{}(x,a) + f^\star(x,a)}} \leq \frac{4
  \K{}\rbr{ \log |\Fcal| + \log(\K{}/\delta)}}{n}.
  \label{eq:tridis_generalization}
\end{align}
\end{theorem}
\pref{thm:plugin_chi_squared} builds on classical convergence results for maximum-likelihood estimators in
well-specified settings, which provide bounds of the form \[
  \EE_{\cD}\brk*{
\Dhelshort^2(\fhatkl{}(x,a) \dmid{} f^\star(x,a))} \leq \bigoh\prn*{\frac{\log
(|\Fcal|/\delta)}{n}}
\] for
any fixed action~\citep[cf.][]{Sara00,zhang2006from}. 
~\pref{thm:plugin_chi_squared} follows quickly from this
classical analysis by applying \pref{prop:hellinger_chi_squared},
which shows that the Hellinger divergence
between Bernoulli distributions upper bounds the \tridis that appears
on the \lhs of \pref{eq:tridis_generalization}.

\pref{thm:plugin_main} immediately follows by combining
\pref{thm:plugin_chi_squared} with the refined Cauchy-Schwarz lemma
(\pref{lem:triangle_selfbounding}) which we restate and prove here.
\triangle*
\begin{proof}[\pfref{lem:triangle_selfbounding}]
  Let $f\in\cF$ be fixed. We first state a simple technical lemma. 
\begin{lemma}
\label{lem:fhat_to_fstar}
For any function $f\in\cF$ and policy $\pi: \Xcal \to \Acal$,
\begin{align*}
\EE_{\cD}\brk{\fstar(x,\pi(x))+f(x,\pi(x))} \leq \En_{\cD}\brk*{\Dtri{\fstar(x,\cdot)}{f(x,\cdot)}} + 4 L(\pi) .
\end{align*}
\end{lemma}
Going forward, define $\gamma(x,a) := f^\star(x,a) -
f(x,a)$ and $s(x,a) := f^\star(x,a) +
f(x,a)$, and $\Delta \ldef
\En_{\cD}\brk*{\Dtri{\fstar(x,\cdot)}{f(x,\cdot)}}$. Let us adopt the
shorthand $\En\equiv\En_{\cD}$. We proceed to bound the cost-sensitive regret:
\begin{align*}
& L(\pi_{f}) - L(\pi^\star) \leq \EE\brk*{ f^\star(x,\pi_{f}(x)) - f(x,\pi_{f}(x)) + f(x,\pi^\star(x)) - f^\star(x,\pi^\star(x))}\\
& \leq \EE \brk*{\sqrt{\frac{\max\crl{s(x,\pi_{f}(x)), s(x,\pi^\star(x))}}{\max\crl{s(x,\pi_{f}(x)), s(x,\pi^\star(x))}}} \cdot \rbr{ |\gamma(x,\pi_{f}(x))| + |\gamma(x,\pi^\star(x))|}}\\
& \leq \sqrt{ \EE \brk*{\max\crl{s(x,\pi_{f}(x)), s(x,\pi^\star(x))}}} \cdot \rbr{ \sum_{\pi \in \{\pi_{f},\pi^\star\}} \sqrt{ \EE \brk*{\frac{ |\gamma(x,\pi(x))|^2}{\max\crl{s(x,\pi_{f}(x)), s(x,\pi^\star(x))}}}}}\\
& \leq \sqrt{\EE\brk*{s(x,\pi_{f}(x)) + s(x,\pi^\star(x)))}} \cdot \rbr{ \sqrt{\EE\brk*{ \frac{\gamma(x,\pi_{f}(x))^2}{s(x,\pi_{f}(x))}}} + \sqrt{\EE\brk*{ \frac{\gamma(x,\pi^\star(x))^2}{s(x,\pi^\star(x))}}}}\\
& \leq \sqrt{\EE\brk*{ (s(x,\pi_{f}(x)) + s(x,\pi^\star(x)))}} \cdot 2 \sqrt{\EE\brk[\bigg]{\sum_a \frac{\gamma(x,a)^2}{s(x,a)}}}\\
& = \sqrt{\EE\brk*{ (s(x,\pi_{f}(x)) + s(x,\pi^\star(x)))}} \cdot 2 \sqrt{\Delta}.
\end{align*}
Here, the first inequality uses that
$f(x,\pi_{f}(x))\leq f(x,\pi^\star(x))$ by the
definition of $\pi_{f}$. The second inequality introduces the $s$
and $\gamma$ quantities, while the third follows from
Cauchy-Schwarz. In the fourth we use that $s(x,\pi_{f}(x))\leq \max
\crl{s(x,\pi_{f}(x)), s(x,\pi^\star(x))}$ and analogously for
$\pi^\star$. Finally we sum over all actions to eliminate the
dependence on the policies to introduce the triangular discrimination $\Delta$. Applying~\pref{lem:fhat_to_fstar},
we additionally observe that
\begin{align*}
\EE \sbr{ s(x,\pi_{f}(x)) + s(x,\pi^\star(x))} \leq 2\Delta + 4\rbr{L(\pi_{f}) + L(\pi^\star)}. %
\end{align*}
After applying standard simplifications, this yields
\begin{align}
L(\pi_{f}) - L(\pi^\star) &\leq 2\sqrt{\Delta} \cdot \sqrt{ 2\Delta + 4 (L(\pi_{f}) + L(\pi^\star))} \leq 2\sqrt{2} \Delta + 4\sqrt{L(\pi^\star) \Delta} + 4\sqrt{L(\pi_{f}) \Delta}\label{eq:cauchy_intermediate}\\
& \leq 6\sqrt{2}\Delta + (L(\pi_{f})+L(\pi^\star))/2.\notag
\end{align}
Re-arranging, we deduce that $L(\pi_{f}) \leq 12\sqrt{2}\Delta + 3
L(\pi^\star)$, and plugging this back into the first inequality in
\pref{eq:cauchy_intermediate} gives
\begin{align*}
L(\pi_{f}) - L(\pi^\star) &\leq 2\sqrt{\Delta} \cdot \sqrt{ 2\Delta + 4 (L(\pi_{f}) + L(\pi^\star))} \leq 2\sqrt{\Delta} \cdot \sqrt{(2+48\sqrt{2})\Delta + 16 L(\pi^\star)}\\
& \leq 8\sqrt{L(\pi^\star) \Delta} + 17\Delta. \tag*\qedhere
\end{align*}
\end{proof}

\begin{proof}[Proof sketch for~\pref{thm:multinomial_plugin}]
The majority of the calculations in this proof are very similar to those
of~\pref{thm:plugin_main}, so we highlight the two main
differences. First, rather than use the \tridis-type bound
in~\pref{thm:plugin_chi_squared}, we use a Hellinger bound on the
maximum likelihood estimate of the multinomial
parameters. Specifically, using essentially the same argument as
in~\pref{thm:plugin_chi_squared}, we can prove that with probability at
least $1-\delta$,
\begin{align*}
\EE_{\cD}\brk*{ \Dhelshort^2(\hat{p}(\cdot \mid x) \dmid{} p^\star(\cdot \mid x))} \leq \frac{2  \log |\Fcal|/\delta}{n},
\end{align*}
where $p^\star(\cdot \mid x) \ldef \PP_{\cD}\brk{y = \cdot \mid x} =
1-f^\star(x,\cdot)$ and $\hat{p}(\cdot \mid x) \ldef 1 -
\fhatkl(x,\cdot)$.

The second change concerns the way we bound the quantity
\[(\fhatkl(x,\pi(x))-\fstar(x,\pi(x)))^{2}/(\fhatkl(x,\pi(x))+\fstar(x,\pi(x))),\] which is done throughout the proof
of~\pref{lem:triangle_selfbounding}.
Rather than naively introduce a sum over all actions
as was done previously, we instead
apply~\pref{prop:hellinger_mult_to_ber}, which relates the multinomial
Hellinger divergence to the \tridis-type quantity above.
\begin{proposition}
\label{prop:hellinger_mult_to_ber}
Let $p,q \in \Delta(\cA)$ be probability mass functions. Then
\begin{align*}
\max_{a \in \Acal} \frac{ (p_a - q_a)^2}{(1-p_a)+ (1-q_a)} \leq 4 \Dhelshort^2(p \dmid{} q).
\end{align*}
\end{proposition}
As a result, for any policy $\pi$
we have
\begin{align*}
\EE_{\cD}\brk*{
  \frac{\fhatkl(x,\pi(x))-\fstar(x,\pi(x)))^2}{\fhatkl(x,\pi(x))+\fstar(x,\pi(x))}}
  &= \EE_{\cD}\brk*{ \frac{ \rbr{p^\star(\pi(x) \mid x) -
    \hat{p}(\pi(x) \mid x)}^2}{(1-p^\star(\pi(x)\mid x)) +
    (1-\hat{p}(\pi(x) \mid x))}} \\
&\leq 2 \EE_{\cD}\brk*{\Dhelshort^2(\hat{p}(\cdot \mid x) \dmid{} p^\star(\cdot \mid x))}.
\end{align*}
All
other calculations are unaffected.
\end{proof}

\subsubsection{Proofs for Supporting Results}
\label{sec:plugin_supplement}

\begin{proof}[Proof of~\pref{prop:hellinger_chi_squared}]
Observe that we can write
\begin{align*}
\Dhelshort^2(p\dmid{}q) = \frac{1}{2}(\sqrt{p} - \sqrt{q})^2 + \frac{1}{2}(\sqrt{1-p} - \sqrt{1-q})^2.
\end{align*}
For each of these terms, we create a difference of squares as follows
\begin{align*}
(\sqrt{x} - \sqrt{y})^2 = \frac{(x-y)^2}{(\sqrt{x}+\sqrt{y})^2} \geq \frac{(x-y)^2}{2(x+y)},
\end{align*}
where the last inequality uses the fact that $2\sqrt{xy} \leq
x+y$. Applying this argument to both terms yields the result.
\end{proof}

\begin{proof}[Proof of~\pref{prop:hellinger_mult_to_ber}]
This is an immediate consequence of the data processing inequality for
Hellinger divergence and~\pref{prop:hellinger_chi_squared}. Indeed, by data processing, we have
\begin{align*}
\Dhelshort^2(p\dmid{}q) \geq \Dhelshort^2 ( (p_a,1-p_a) \dmid{} (q_a,1-q_a)),
\end{align*}
since the latter is the distribution of the random variable $Y :=
\one\{X=a\}$ when $X \sim p$ (resp. $q$). Now that we have passed to
the Bernoulli Hellinger divergence, we simply
apply~\pref{prop:hellinger_chi_squared} and drop one of the two terms.
\end{proof}

\iffalse
The desired inequality is equivalent to
\begin{align*}
\EE\sqrt{AB} + \sqrt{(1-A)(1-B)} \leq \exp\rbr{ - \EE \frac{(A-B)^2}{4(A+B)}}.
\end{align*}
We prove this inequality by upper bounding the \lhs. First observe that
\begin{align*}
\sqrt{(1-A)(1-B)} \leq 1 - (A+B)/2.
\end{align*}
This can be verified by squaring both sides and re-arranging terms to
find a quadratic form. Using this elementary inequality, we have
\begin{align*}
\EE\sqrt{AB} + \sqrt{(1-A)(1-B)} & \leq 1 + \EE\sqrt{AB} - (A+B)/2 = 1 - \EE (\sqrt{A} - \sqrt{B})^2/2\\
& \leq \exp\rbr{ - \EE\frac{ (\sqrt{A} - \sqrt{B})^2}{2}},
\end{align*}
where in the last step we use that $1-x \leq \exp(-x)$. 
Next, observe that
\begin{align*}
(\sqrt{A} - \sqrt{B})^2 = \frac{(\sqrt{A} - \sqrt{B})^2(\sqrt{A} + \sqrt{B})^2}{(\sqrt{A} + \sqrt{B})^2} = \frac{ (A-B)^2 }{(\sqrt{A} + \sqrt{B})^2} \geq \frac{ (A-B)^2}{2(A+B)},
\end{align*}
where the final bound uses the fact that $(\sqrt{x}+\sqrt{y})^2 \leq 2(x+y)$. 
Plugging this in yields the proposition.
\end{proof}
\fi

\begin{proof}[\pfref{thm:plugin_chi_squared}]
\newcommand{\Loss}{C}
The initial steps of this proof parallel the classical analysis of
maximum likelihood estimators~\citep[see, e.g.,][]{zhang2006from}. We
start by establishing a symmetrization inequality. Let $D := \{(x_i,\ls_i)\}_{i=1}^n$ and $D':= \{(x'_i,\ls'_i)\}_{i=1}^n$
denote two \iid datasets of $n$ examples, let $\Loss(f,D)$ be any function of a regression function $f$ and dataset $D$,
and let $\hat{f}$ be any estimator that takes the dataset $D$ and outputs a
function in $\Fcal$. We first show that
\begin{align}
\EE_D\sbr{ \exp\prn*{\Loss(\hat{f}(D), D) - \log \EE_{D'}\brk*{\exp(\Loss(\hat{f}(D), D')}} - \log |\Fcal|)} \leq 1. \label{eq:plugin_chernoff}
\end{align}
This is a symmetrization inequality because it relates the ``training
error'' $\Loss(\hat{f}(D), D)$ to the error $\Loss(\hat{f}(D), D')$
measured on the ``ghost sample'' $D'$. The unusual form of the expression involving the
ghost sample is to accommodate the fact that $\Loss$ may be unbounded.

To prove~\pref{eq:plugin_chernoff}, let $\mu$ denote the uniform
distribution over $\Fcal$, and observe that for any distribution
$\hat{\mu} \in \Delta(\Fcal)$ and any function $g: \Fcal \to \RR$,
we have
\begin{align*}
\sum_{f \in \Fcal} \hat{\mu}(f) g(f) \leq \max_{f \in \Fcal} g(f) \leq
\log\sum_{f\in\Fcal} \exp(g(f)) = \log \rbr{\EE_{f \sim \mu} \exp(g(f))} + \log |\Fcal|.
\end{align*}
Now for any $D$ we take $\hat{\mu}(f) := \one\{f = \hat{f}(D)\}$ and $g(f) := \Loss(f,D) - \log
\EE_{D'} \exp(\Loss(f,D'))$ to obtain
\begin{align*}
\Loss(\hat{f}(D),D) - \log \EE_{D'} \exp(\Loss(\hat{f}(D),D')) \leq \log \rbr{ \EE_{f\sim\mu} \frac{\exp(\Loss(f,D))}{\EE_{D'}\exp(\Loss(f,D'))}} + \log |\Fcal|.
\end{align*}
We will exponentiate this inequality and take expectation over the
initial dataset $D$. When we do this, the first term on the \rhs simplifies to
\begin{align*}
\EE_{D} \exp\rbr{\log \rbr{ \EE_{f\sim\mu}
  \brk*{\frac{\exp(\Loss(f,D))}{\EE_{D'}\exp(\Loss(f,D'))}}}} = \EE_{f \sim
  \mu} \brk*{\frac{\EE_D \exp(\Loss(f,D))}{\EE_{D'} \exp(\Loss(f,D'))}} = 1.
\end{align*}
Re-arranging, we obtain~\pref{eq:plugin_chernoff}. With the
exponential moment bound in~\pref{eq:plugin_chernoff}, a standard
application of the Chernoff method yields that for any $\delta \in
(0,1)$ with probability at least $1-\delta$ we have
\begin{align*}
- \log \EE_{D'}\exp(\Loss(\hat{f}(D),D')) \leq - \Loss(\hat{f}(D),D) + \log |\Fcal| + \log(1/\delta).
\end{align*}
This high-probability bound holds for any fixed functional $\Loss$. To
apply it, for each $a\in\cA$, we define
\begin{align*}
\Loss_a(f,D) := -\frac{1}{2} \sum_{i=1}^{n} \ls_i(a) \log
  (f^\star(x_i,a)/f(x_i,a)) + (1-\ls_i(a))\log (
  (1-f^\star(x_i,a))/(1-f(x_i,a))), 
\end{align*}
where $y_i(a)$ is defined as in \pref{eq:plugin_erm}. We apply the
bound for each $C_a$, then take a union bound over all $a\in\cA$ and
sum up the resulting inequalities, which gives that with
probability at least $1-\delta$,
\begin{align*}
\sum_{a \in \Acal} - \log \EE_{D'}\exp(\Loss_a(\hat{f}(D),D')) \leq
  \sum_{a \in \Acal} -
  \Loss_a(\hat{f}(D),D) + \K{} \rbr{ \log |\Fcal| + \log(\K{}/\delta)}.
\end{align*}
If we apply this inequality with $\fhatkl$ as the maximum likelihood estimate, we
have $ \sum_{a} -\Loss_a(\fhatkl(D),D) \leq 0$. On the other hand, for
each action $a \in \Acal$, the corresponding term on the \lhs can
be simplified to
\begin{align*}
& -\log \EE_{D'} \exp\rbr{ -\frac{1}{2} \sum_{i=1}^n\rbr{ \ls'_i(a)
                 \log \frac{f^\star(x'_i,a)}{\fhatkl{}(x'_i,a)} +
                 (1-\ls'_i(a)) \log
                 \frac{1-f^\star(x_i',a)}{1-\fhatkl{}(x_i',a)}}}\\
  \intertext{Now, let $y'_i(a)\sim\mathrm{Ber}(\ls'_i(a))$. Then by
  Jensen's inequality, we have}
& \geq{} - n \log \EE_{x',\ls'}\En_{y'\mid{}\ls'} \exp\rbr{ -\frac{1}{2}\rbr{ y'(a) \log \frac{f^\star(x',a)}{\fhatkl{}(x',a)} + (1-y'(a)) \log \frac{1-f^\star(x',a)}{1-\fhatkl{}(x',a)}}}\\
& = - n \log \EE_{x',\ls'}\En_{y'\mid{}\ls'} \brk*{\rbr{\frac{f^\star(x',a)}{\fhatkl{}(x',a)}}^{-y'(a)/2}\rbr{\frac{1-f^\star(x',a)}{1-\fhatkl{}(x',a)}}^{-(1-y'(a))/2}}\\
& = - n \log \EE_{x'} \brk*{\sqrt{f^\star(x',a)\fhatkl{}(x',a)} + \sqrt{(1-f^\star(x',a))(1-\fhatkl{}(x',a))}}.
\end{align*}
Here the last line holds because the model is well-specified; in
particular $\PP\brk{y'(a) = 1 \mid x'} = f^\star(x',a)$. Continuing,
observe that for any random variables $u,v$ taking values in $[0,1]$
we have
\begin{align}
\label{eq:log_hellinger}
- \log \EE\brk*{\sqrt{uv} + \sqrt{(1-u)(1-v)}} = - \log \rbr{ 1 - \EE\sbr{ 1 - \sqrt{uv} - \sqrt{(1-u)(1-v)}}} \geq \frac{1}{2} \EE\brk*{ \Dhelshort^2(u\dmid{}v)},
\end{align}
where the last step uses that $x \leq -\log(1-x)$ for $x\in\brk*{0,1}$ along with the
definition of the Hellinger divergence.
Together, these inequalities establish that
\begin{align*}
\frac{1}{2} \sum_{a \in \Acal}\En_x\brk*{ \Dhelshort^2(f^\star(x,a) \dmid{} \fhatkl(x,a))} 
\leq
\frac{\K{}\rbr{ \log|\Fcal| + \log(\K{}/\delta)}}{n}.
\end{align*}
To conclude, we simply apply~\pref{prop:hellinger_chi_squared}, which yields the result. 
\end{proof}

\begin{proof}[Proof of~\pref{lem:fhat_to_fstar}]
  Let $f\in\cF$ be fixed and define $\gamma(x,a) := f^\star(x,a) -
f(x,a)$ and $s(x,a) := f^\star(x,a) +
f(x,a)$. By the triangle inequality, the AM-GM inequality, and an application
of~\pref{thm:plugin_chi_squared}, we have
\begin{align*}
\EE_{\cD}\brk*{s(x,\pi(x))} &\leq \EE_{\cD}|\gamma(x,\pi(x))| + 2 L(\pi^\star)\\
& \leq \EE_{\cD} \brk*{\sqrt{s(x,\pi(x))} \frac{|\gamma(x,\pi(x))|}{\sqrt{s(x,\pi(x))}}} + 2 L(\pi^\star)\\
& \leq \frac{1}{2}\EE_{\cD}\brk*{s(x,\pi(x))} + \frac{1}{2} \EE_{\cD}\brk*{\frac{\gamma(x,\pi(x))^2}{s(x,\pi(x))}} + 2L(\pi^\star)\\
& \leq \frac{1}{2}\EE_{\cD}\brk*{s(x,\pi(x))} + \frac{1}{2} \EE_{\cD}\brk*{\sum_a \frac{\gamma(x,a)^2}{s(x,a)}} + 2 L(\pi^\star)\\
& \leq \frac{1}{2}\EE_{\cD}\brk*{s(x,\pi(x))} +\frac{1}{2}\En_{\cD}\brk*{\Dtri{\fstar(x,\cdot)}{f(x,\cdot)}}+ 2 L(\pi^\star).
\end{align*}
Re-arranging yields the result.
\end{proof}

\arxiv{\section{Proofs for Contextual Bandit
  Results (\pref{sec:cb})}}
\neurips{\section{Proofs and Additional Discussion for Contextual Bandits Results}}
\label{app:cb}

\subsection{Online Regression Oracles}
\label{app:oracle}
In this section we briefly formalize the notion of an online regression oracle
sketched in the introduction and \pref{ass:logloss_regret}. The treatment here follows \citet{foster2020beyond}.

We consider the following model for the oracle \algtext.
\begin{itemize}
\item[] For $t=1,\ldots,T$:
  \begin{itemize}
  \item Nature selects context-action pair $(x_t,a_t)\in\cX\times\cA$.
  \item Algorithm produces prediction $\yh_t\in\brk*{0,1}$.
  \item Nature selects outcome $y_t\in\brk*{0,1}$.
  \end{itemize}
\end{itemize}
We model the oracle as a sequence of mappings
$\alg\ind{t}:(\cX\times\cA)\times\prn*{\cX\times\cA\times{}\bbR}^{t-1}\to\brk*{0,1}$, so that
$\yh_t=\alg\ind{t}\prn*{x_t,a_t\midsem\crl{(x_i,,a_i,y_i)}_{i=1}^{t-1}}$
above. Any algorithm of this type induces a mapping
\begin{equation}
\yh_t(x,a)
\ldef{}\alg\ind{t}\prn*{x,a\midsem\crl{(x_i,,a_i,y_i)}_{i=1}^{t-1}},\label{eq:yhat}
\end{equation}
which may be understood as the prediction the algorithm would make at time
$t$ if we froze its internal state and selected $(x_t,a_t)=(x,a)$.

\neurips{
\subsection{Examples of Oracles for \mainalg}
\label{sec:cb_examples}
\neurips{In this section we}\arxiv{We now} take advantage of the extensive literature on regression with the
logarithmic loss
\citep{cover1991universal,vovk1995game,kalai2002efficient,hazan2015online,orseau2017soft,rakhlin2015sequential,foster2018logistic,luo2018efficient}
and instantiate
\pref{thm:main} to give provable and efficient first-order
regret bounds for a number of function classes of interest. To the
best of our knowledge, our results are new for each of these special
cases.

\begin{example}[Finite function classes]
  If $\cF$ is a finite class, Vovk's aggregating algorithm \citep{vovk1995game} guarantees
  that\footnote{See \pref{prop:exp_concave} for a proof that the loss
    $\logl(\yhat,y)$ is mixable over the domain $\brk*{0,1}$, which is required to apply this result.}
  \begin{equation}
    \label{eq:vovk}
    \RegLog \leq \log\abs{\cF}.
  \end{equation}
  With this choice, $\mainalg$ satisfies $\En\brk{\RegCB} \leq \bigoh\prn*{\sqrt{\Lstar\cdot\K{}\log\abs{\cF}} + \K{}\log\abs{\cF}}$.
\end{example}

\begin{example}[Low-dimensional linear functions]
  Suppose that $\cF$ takes the form
  \[
\cF = \crl*{(x,a)\mapsto{} \tri*{w,\phi(x,a)}\mid{}w\in\Delta_{d}},
\]
where $\phi(x,a)\in\bbR_{+}^{d}$ is a fixed feature map with
$\nrm*{\phi(x,a)}_\infty\leq{}1$. Then the continuous exponential weights
algorithm ensures that
\[
\RegLog \leq{} \bigoh(d\log(T/d)),
\]
and can be implemented in $\mathrm{poly}(d,T)$ time per step using
log-concave sampling
\citep{cover1991universal,kalai2002efficient}. With this choice,
\mainalg satisfies
\begin{equation}
  \label{eq:linear}
  \En\brk*{\RegCB} \leq{}\bigoh\prn*{\sqrt{\Lstar\cdot\K{}d\log(T/d)} + \K{}d\log(T/d)}.
\end{equation}
\end{example}
Beyond attaining first-order regret, this bound in \pref{eq:linear} is minimax optimal
when the number of actions is constant \citep{li2019nearly}. A natural
direction for future work is to improve the result for large action spaces. Another
more practical choice for the oracle in this setting is the algorithm
of \citet{luo2018efficient}, which has slightly worse regret
$\RegLog\leq\bigoht(d^2)$, but runs in time $\bigoh(Td^{2.5})$ per step.

While first-order regret bounds for contextual bandits have primarily been
investigated for finite classes prior to this work, an advantage of
working within the regression oracle framework is that we can easily
lift our first-order guarantees to rich, nonparametric function classes.
\begin{example}[High/infinite-dimensional linear functions]
  \label{ex:highdim}
  Suppose that $\cF$ takes the form
  \[
\cF = \crl*{(x,a)\mapsto{} \tfrac{1}{2}(1+\tri*{w,\phi(x,a)})\mid{}\nrm*{w}_{2}\leq{}1},
\]
where $\nrm*{\phi(x,a)}_2\leq{}1$ is a fixed feature map. For this
setting, \citet[Section 6.1]{rakhlin2015sequential} show that the
follow-the-regularized-leader algorithm with
log-barrier regularization has\footnote{This is technically
  only proven for the case where $y\in\crl*{0,1}$, but the proof easily extends to $y\in\brk*{0,1}$.}
\[
\RegLog \leq{} \bigoh(\sqrt{T\log(T)}).
\]
This algorithm can be implemented in time $\bigoh(d)$ per step. For
this choice, \mainalg satisfies the dimension-independent rate
\begin{equation}
  \En\brk*{\RegCB} \leq{}\bigoh\prn*{(\K\Lstar)^{1/2}T^{1/4} +
    \K\sqrt{T}},
  \label{eq:example_highdim}
\end{equation}
\end{example}
Let us interpret the bound in \pref{eq:example_highdim}. First, we recall that the minimax optimal
rate for this function class is $\K^{1/2}T^{3/4}$, which the bound
above always achieves in the worst case
\citep{abe1999associative,foster2020beyond}; this
``worse-than-$\sqrt{T}$'' rate is the price we pay for working with an
expressive function class. On the other hand, if
$\Lstar$ is constant the bound in \pref{eq:example_highdim} improves
to $\bigoh(\K\sqrt{T})$, which beats the worst-case rate. While one
might hope that a tighter rate of the form, e.g.,
$(\Lstar)^{3/4}$, might be possible, by adapting
a lower bound in \citet[Section 4]{srebro2010smoothness}, one can show that the
result in \pref{eq:example_highdim} cannot be improved.

\begin{example}[Kernels]The algorithm in \pref{ex:highdim} kernelizes
  and hence can be immediately applied when
  $\cF=\crl{(x,a)\mapsto{}\frac{1}{2}(1+g(x,a))\mid{}g\in\cG}$, where
  $\cG$ is  reproducing kernel space with RKHS norm
  $\nrm*{\cdot}_{\cG}$ and kernel $\cK$. The regret bound in
  \pref{eq:example_highdim} continues to hold for this setting as
  long as
  $\nrm*{g}_{\cG}\leq{}1$ and $\cK((x,a), (x,a)) \leq{} 1$.
\end{example}
The logarithmic loss is also well-suited to generalized
linear models, as the following example highlights.%
\begin{example}[Generalized linear models]
  \label{ex:logistic}
Let
$\cF=\crl*{(x,a)\mapsto\sigma(\tri{w,\phi(x,a)})\mid{}w\in\bbR^{d},
  \nrm*{w}_2\leq{}1}$, where $\sigma(t)=1/(1+e^{-t})$ is the logistic
link function and $\phi(x,a)$ is a fixed feature map. In this case,
the map $w\mapsto{}\logloss(\sigma(\tri{w,\phi(x,a)}),y)$ is
equivalent to the
standard logistic loss function applied to $\tri{w,\phi(x,a)}$, and we can use the algorithm from
\citet{foster2018logistic} to obtain $\RegLog\leq\bigoh(d\log(T/d))$
and $\RegCB\leq\bigoht(\sqrt{\Lstar\cdot{}Ad} + Ad)$. When $d$ is
large, we can also use online gradient descent on the logistic
loss, which gives $\RegLog\leq\bigoh(\sqrt{T})$ and $\En\brk*{\RegCB} \leq{}\bigoh\prn*{(\K\Lstar)^{1/2}T^{1/4} +
    \K\sqrt{T}}$.
\end{example}

Beyond the algorithmic examples above, for general function classes \citet{bilodeau2020tight}
provide a tight characterization for the minimax optimal rates
for online regression with the logarithmic loss in terms of
\emph{sequential covering numbers} \citep{rakhlin2015sequential} for
the class $\cF$. We can use these in tandem with
\pref{thm:main} to give new regret bounds for general
classes. For example, when $\cF$ is the set of all $\brk*{0,1}$-valued
$1$-Lipschitz functions over $\brk*{0,1}^{d}$,
\citet{bilodeau2020tight} show that the optimal rate for \loglosst
regression is $\RegLog=\Theta(T^{\frac{d}{d+1}})$, which gives $\RegCB\approxleq
\bigoh\prn*{(\Lstar)^{1/2}T^{\frac{d}{2(d+1)}}+T^{\frac{d}{d+1}}}$ for
\mainalg.\neurips{\looseness=-1}

 }

\subsection{Proof of \pref{thm:main}}

\maintheorem*

\begin{proof}%
  Define $\Lhat_T=\sum_{t=1}^{T}\ls_t(a_t)$ and
  $\Lstar_T=\sum_{t=1}^{T}\ls_t(\pistar(x_t))$. All of the effort in this proof will be to show that for any choice $\gamma\geq{}10\K{}$, \pref{alg:online_main} has
 \begin{equation}
   \label{eq:online_oracle}
   \En\brk*{\RegCB}\leq 
\frac{10\K{}}{\gamma}\En\brk*{\Lstar_T} +
28\gamma\cdot{}\RegLog.
\end{equation}
The bound in \pref{eq:mainalg_regret} immediately follows from this
guarantee by using choice of $\gamma$ in the theorem statement.

Define a filtration    \begin{equation}
      \filt_{t-1}=\sigma((x_1,a_1,\ls_1(a_1)), \ldots,
      (x_{t-1},a_{t-1},\ls_{t-1}(a_{t-1})), x_t)\label{eq:filtration1}
    \end{equation}
    and let
    $\En_{t}\brk*{\cdot}\ldef\En\brk*{\cdot\mid{}\filt_t}$. Next, 
    define the following conditional-expected versions of the
    contextual bandit regret and \loglosst regret, respectively
    \[
      \RegBarCB =
      \sum_{t=1}^{T}\En_{t-1}\brk*{\ls_t(a_t)-\ls_t(\pistar(x_t))}
      = \sum_{t=1}^{T}\sum_{a}p_{t,a}(\fstar(x_t,a)-\fstar(x_t,\pistar(x_t)))
    \]
    and
    \[
      \RegBarLog =
      \sum_{t=1}^{T}\En_{t-1}\brk*{\logl(\yhat_{t}(x_t,a_t),
        \ls_t(a_t))- \logl(\fstar(x_t,a_t),\ls_t(a_t))}.
    \]
    Our starting point is to observe that
    $\En\brk*{\RegCB}=\En\brk*{\RegBarCB}$ and 
    $\En\brk*{\RegBarLog}\leq\RegLog$, where the latter holds since $\RegLog$ is a deterministic upper bound on the \loglosst regret of the oracle. So it suffices to relate the
    conditional-expected versions of these quantities.

The main step of the proof is to upper bound $\RegBarCB$, using the
first-order per-round inequality \pref{thm:per_round} (proven in \pref{app:per_round}), which we restate here
for completeness.
\perround*
    Applying \pref{thm:per_round} for each round $t$, we are guaranteed that
    \begin{align*}
\RegBarCB &\leq
\frac{5\K{}}{\gamma}\sum_{t=1}^{T}\sum_{a}p_{t,a}\fstar(x_t,a) + 7\gamma
            \sum_{t=1}^{T}\sum_{a}p_{t,a}\frac{(\yhat_t(x_t,a)-\fstar(x_t,a))^{2}}{\yhat_t(x_t,a)+\fstar(x_t,a)}\\
          &= \frac{5\K{}}{\gamma}\Lhatb_T + 7\gamma\cdot{}\ErrDelBar,
    \end{align*}
    where $\Lhatb_T\ldef \sum_{t=1}^{T}\sum_{a}p_{t,a}\fstar(x_t,a)$ and
    \[
      \ErrDelBar \ldef \sum_{t=1}^{T}\sum_{a}p_{t,a}\frac{(\yhat_t(x_t,a)-\fstar(x_t,a))^{2}}{\yhat_t(x_t,a)+\fstar(x_t,a)}.
    \]
Next, we relate the \tridis-type error $\ErrDelBar$ to the \loglosst regret using the following
proposition (proven in the sequel).
\begin{proposition}
  \label{prop:kl_calibrated}
  If $y\in\brk{0,1}$ is a random variable with $\En\brk*{y}=\mu$, then for any $\yhat\in\brk{0,1}$,
  \begin{equation}
    \label{eq:kl_calibrated}
    \En\brk*{\logl(\yhat,y)-\logl(\mu,y)} =\dkl{\mu}{\yhat} \geq{} \frac{1}{2}\cdot\frac{(\yhat-\mu)^{2}}{\yhat+\mu}.
  \end{equation}
\end{proposition}
In particular, since $a_t$ and $\ls_t$ are conditionally independent given $\filt_{t-1}$, this implies
that
\[
\ErrDelBar \leq{}
2\sum_{t=1}^{T}\sum_{a}p_{t,a}\dkl{\fstar(x_t,a)}{\yhat_t(x,a_t)} = 2\RegBarLog,
\]
so that
\[
\RegBarCB \leq 
\frac{5\K{}}{\gamma}\Lhatb_T + 14\gamma\cdot{}\RegBarLog.
\]
To conclude, let
$\Lstarb_T=\sum_{t=1}^{T}\fstar(x_t,\pistar(x_t))$. Then this
inequality can be written as
\[
\Lhatb_T - \Lstarb_T \leq 
\frac{5\K{}}{\gamma}\Lhatb_T + 14\gamma\cdot{}\RegBarLog.
\]
  Since $1/(1-\veps)\leq{}1+2\veps$ for all $\veps\leq{}1/2$, this implies
  that whenever $\gamma{}\geq{}10\K{}$,
  \[
\Lhatb_T - \Lstarb_T \leq 
\frac{10\K{}}{\gamma}\Lstarb_T + 28\gamma\cdot{}\RegBarLog.
\]
Noting that $\En\brk{\Lstarb_T}=\En\brk*{\Lstar_T}$ and $\En\brk{\Lhatb_T}=\En\brk{\Lhat_T}$, this establishes \pref{eq:online_oracle}.

\end{proof}

\begin{proof}[\pfref{prop:kl_calibrated}]
  For the equality in \pref{eq:kl_calibrated}, we have
  \begin{align*}
    \En\brk*{\logl(\yhat,y)-\logl(\mu,y)} = \En\brk*{y\log(\mu/\yhat) + (1-y)\log((1-\mu)/(1-\yhat))}=\dkl{\mu}{\yhat}.
  \end{align*}
  To prove the inequality, let $f_{\yhat}(\mu) = \dkl{\mu}{\yhat}$. By Taylor's theorem, we
  have
  \[
    f_{\yhat}(\mu) = f_{\yhat}(\yhat) + f_{\yhat}'(\yhat)(\mu-\yhat)
    + \frac{1}{2}f_{\yhat}''(\bar{y})(\mu-\yhat)^{2},
  \]
  for some $\bar{y}\in\mathrm{conv}(\crl{\yhat,\mu})$. Observe that
  \[
  f'_{\yhat}(z) = \log(z/\yhat)  - \log((1-z)/(1-\yhat)),
  \]
  so that we have
  $f_{\yhat}(\yhat)=f_{\yhat}'(\yhat)=0$. Further
  \[
    f''_{\yhat}(\bar{y}) = \frac{1}{\bar{y}} + \frac{1}{1-\ybar}
    \geq{} \frac{1}{\max\crl{\yhat,\mu}}\geq\frac{1}{\yhat+\mu},
  \]
  which establishes the result.
\end{proof}

\subsection{Proof of \pref{thm:per_round}}
\label{app:per_round}
\perround*
\begin{proof}%
  To begin, we observe that by the AM-GM inequality,
  \begin{align}
    \notag
    \sum_{a}p_a(f_a-f_{\astar})
    &= \sum_{a\neq\astar}p_a(y_a-f_{\astar})
    + \sum_{a\neq\astar}p_a(f_a-y_a)\\
    &\leq{} \sum_{a\neq\astar}p_a(y_a-f_{\astar}) +
      \frac{1}{4\gamma}\sum_{a\neq\astar}p_a(f_a+y_a)
      + \gamma\sum_{a\neq\astar}p_a\frac{(y_a-f_a)^{2}}{y_a+f_a}.
      \label{eq:per_round_basic}
  \end{align}
  We focus on bounding the first term in \pref{eq:per_round_basic},
  then return to the other terms at the end of the proof. We have
\begin{align}
\sum_{a \ne a^\star} p_a (y_a - f_{a^\star}) &= \sum_{a \ne a^\star} p_a (y_a - y_b) + (1-p_{a^\star})(y_b - f_{a^\star})\notag\\
& = \sum_{a \notin \{a^\star, b\}} p_a (y_a - y_b) + (1-p_{a^\star})(y_b - f_{a^\star}). \label{eq:per_round_2}
\end{align}
Recall that for $a \ne b$ we set $p_a = \frac{y_b}{\K{}y_b
  + \gamma (y_a - y_b)}$ and for $p_b$ we set $p_b = 1 - \sum_{a \ne
  b} p_a$. With this setting, the first term in \pref{eq:per_round_2} is bounded as
\begin{align}
\sum_{a \notin \{a^\star, b\}} p_a (y_a - y_b) \leq \sum_{a \notin
  \{a^\star,b\}} \frac{y_b (y_a - y_b)}{\K{} y_b + \gamma(y_a - y_b)}
  \leq \K{}\frac{y_b}{\gamma}.
  \label{eq:per_round_not_ab}
\end{align}
It remains to bound the term
\[
(1-p_{a^\star})(y_b - f_{a^\star}).
\]
If $f_{a^\star} \geq y_b$ this is trivially negative, so we assume going forward that
$f_{a^\star} \leq y_b$, and upper bound as
\[
(1-p_{a^\star})(y_b - f_{a^\star}) \leq y_b - f_{a^\star}.
\]
We now appeal to the following lemma.
\begin{lemma}
  \label{lem:yb_fa}
  The distribution $p$ in \pref{thm:per_round} ensures that
  \begin{align}
    \label{eq:yb_fa}
    y_b - f_{a^\star} \leq{} \frac{\K{}}{4\gamma}y_{b} +   2\gamma\cdot{}p_{\astar}\frac{(y_{\astar}-f_{\astar})^{2}}{y_{\astar}+f_{\astar}}.
  \end{align}
\end{lemma}
Combining \pref{eq:per_round_basic}, \pref{eq:per_round_not_ab}, and
\pref{eq:yb_fa}, we arrive at the bound.
\begin{align}
  \label{eq:per_round_main}
  \sum_{a}p_a(f_a-f_{\astar}) 
  \leq{} \frac{1}{4\gamma}\sum_{a}p_a(f_a+y_a) +
  2\gamma\sum_{a}p_a\frac{(y_a-f_a)^{2}}{y_a+f_a}
  + \frac{2\K{}}{\gamma}y_{b}.
\end{align}
To conclude, we relate the non-triangular terms above to
$\sum_{a}p_af_a$, which corresponds to the learner's expected loss. For the first
term, we use the following basic result.
\begin{lemma}
  \label{lem:y_to_f}
For any distribution $p\in\Delta_{A}$, 
  \begin{align*}
    \sum_{a}p_ay_a
   \leq{} 
    3\sum_{a}p_af_a + \sum_{a}p_a\frac{(y_a-f_a)^{2}}{y_a+f_a}.
  \end{align*}
\end{lemma}
Applying this gives
\begin{align*}
  \sum_{a}p_a(f_a-f_{\astar}) 
  \leq{} \frac{1}{\gamma}\sum_{a}p_af_a +
  3\gamma\sum_{a}p_a\frac{(y_a-f_a)^{2}}{y_a+f_a}
  + \frac{2\K{}}{\gamma}y_{b},
\end{align*}
where we have used that $\gamma\geq{}1$ to simplify. Our final step is
to relate the last term above to $f_{\astar}$. To do this, we observe
that if $\gamma\geq{}2\K{}$, then \pref{lem:yb_fa} implies (after
rearranging), that
\[
y_b \leq{} 2f_{a^\star} +   4\gamma\cdot{}p_{\astar}\frac{(y_{\astar}-f_{\astar})^{2}}{y_{\astar}+f_{\astar}},
\]
so that
\[
\frac{2\K{}}{\gamma}y_{b} \leq{} \frac{4\K{}}{\gamma}f_{\astar} +
8\K{}p_{\astar}\frac{(y_{\astar}-f_{\astar})^{2}}{y_{\astar}+f_{\astar}}
\leq{} \frac{4\K{}}{\gamma}f_{\astar} +
4\gamma\cdot{}p_{\astar}\frac{(y_{\astar}-f_{\astar})^{2}}{y_{\astar}+f_{\astar}}.
\]
With this, we have
\[
    \sum_{a}p_a(f_a-f_{\astar}) 
  \leq{} \frac{1}{\gamma}\sum_{a}p_af_a +
  7\gamma\sum_{a}p_a\frac{(y_a-f_a)^{2}}{y_a+f_a}
  + \frac{4\K{}}{\gamma}f_{\astar},
\]
Finally, since $\astar\in\argmin_a{f_{a}}$, we have
$f_{\astar}\leq{}\sum_{a}p_af_a$, so we can simplify to
\[
    \sum_{a}p_a(f_a-f_{\astar}) 
  \leq{} \frac{5\K{}}{\gamma}\sum_{a}p_af_a +
  7\gamma\sum_{a}p_a\frac{(y_a-f_a)^{2}}{y_a+f_a}.
\]
\end{proof}

\subsubsection{Proofs for Supporting Lemmas}
\label{app:per_round_supporting}

\begin{proof}[\pfref{lem:yb_fa}]
Assume that $y_b\geq{}f_{\astar}$, or else we are done. We consider two cases.
  \paragraph{Case 1: $\astar=b$}
  In this case, by the AM-GM inequality
  \[
    y_b - f_{a^\star} = y_{\astar} - f_{a^\star} \leq{}
    \frac{y_{\astar}+f_{\astar}}{8\gamma{}p_{\astar}} +
    2\gamma{}\cdot{}p_{\astar}\frac{(y_{\astar}-f_{\astar})^{2}}{y_{\astar}+f_{\astar}}.
  \]
  Since $\astar=b$, we have
  \begin{align*}
    p_{a^\star} = p_b = 1 - \sum_{a \ne b} \frac{y_b}{\K{}y_b + \gamma (y_a - y_b)} \geq 1/\K{},
  \end{align*}
  so we can further upper bound by
  \begin{equation*}
    \frac{\K{}}{8\gamma}(y_{\astar}+f_{\astar}) +
    2\gamma\cdot{}p_{\astar}\frac{(y_{\astar}-f_{\astar})^{2}}{y_{\astar}+f_{\astar}}
    \leq{}   \frac{\K{}}{4\gamma}y_{\astar}+
    2\gamma\cdot{}p_{\astar}\frac{(y_{\astar}-f_{\astar})^{2}}{y_{\astar}+f_{\astar}}
    = \frac{\K{}}{4\gamma}y_{b}+
    2\gamma\cdot{}p_{\astar}\frac{(y_{\astar}-f_{\astar})^{2}}{y_{\astar}+f_{\astar}},
  \end{equation*}
  where we have used that $f_{\astar}\leq{}y_b = y_{\astar}$, where the latter holds since, for this case, we are assuming $\astar=b$.

\paragraph{Case 2: $\astar\neq{}b$}
Observe that in this case, we have
\begin{equation}
  \label{eq:casetwo_conditions}
  y_{\astar}\geq{}y_b,\quad\text{and}\quad{}f_{b}\geq{}f_{\astar}.
\end{equation}
Since $\astar\neq{}b$, using the definition of $p_{\astar}$, we have
\begin{align*}
  y_{b}-f_{\astar} &=
                     p_{\astar}\frac{\K{}{}y_b+\gamma(y_{\astar}-y_b)}{y_b}(y_b-f_{\astar})\\
                   &= \K{}p_{\astar}(y_b-f_{\astar}) + \gamma{}\cdot{}p_{\astar}\frac{(y_{\astar}-y_b)(y_b-f_{\astar})}{y_b},
\end{align*}
which we can rewrite as 
\begin{align*}
    y_{b}-f_{\astar} &= \underbrace{\K{}p_{\astar}(y_b-f_{\astar}) - \gamma\cdot{}p_{\astar}\frac{(y_b-f_{\astar})^{2}}{y_b}}_{\textbf{A}}
                       + \underbrace{\gamma\cdot{}p_{\astar}\frac{(y_{\astar}-f_{\astar})(y_b-f_{\astar})}{y_b}}_{\textbf{B}}.
\end{align*}
For the term $\textbf{A}$ above, we observe that by the AM-GM inequality,
\begin{align}
  \label{eq:casetwo_partone}
  \K{}p_{\astar}(y_b-f_{\astar}) 
  \leq{} \frac{\K{}^{2}}{4\gamma}p_{\astar}y_b + \gamma{}p_{\astar}\frac{(y_b-f_{\astar})^{2}}{y_b},
\end{align}
so that
\[
\textbf{A} \leq{} \frac{\K{}^{2}}{4\gamma}p_{\astar}y_b \leq{} \frac{\K{}}{4\gamma}y_b,
\]
where we have used that $p_{\astar}\leq{}1/\K{}$ when $\astar\neq{}b$.

Next, to bound $\textbf{B}$, we observe that
$y_{\astar}\geq{}y_b\geq{}f_{\astar}\geq{}0$. Since the function
$a\mapsto\frac{(a-b)}{a}$ is increasing for $a,b\geq{}0$, we have
that
$\frac{(y_b-f_{\astar})}{y_{b}} \leq{}
\frac{(y_{\astar}-f_{\astar})}{y_{\astar}}$ and consequently
\[
  \frac{(y_{\astar}-f_{\astar})(y_b-f_{\astar})}{y_b} \leq{}
  \frac{(y_{\astar}-f_{\astar})^{2}}{y_{\astar}} \leq{}
  2\frac{(y_{\astar}-f_{\astar})^{2}}{y_{\astar}+f_{\astar}},
\]
where the second inequality uses that
$y_{\astar}\geq{}f_{\astar}$.

Altogether, we have that when $\astar\neq{}b$,
\begin{align}
  \label{eq:casetwo_conclusion}
  y_{b}-f_{\astar} =\textbf{A}+\textbf{B} \leq \frac{\K{}}{4\gamma}y_{b} + 2\gamma{}\cdot{}p_{\astar}  \frac{(y_{\astar}-f_{\astar})^{2}}{y_{\astar}+f_{\astar}}.
\end{align}
The result now follows by combining the two cases.
\end{proof}

\begin{proof}[\pfref{lem:y_to_f}]
  First, we write
  \begin{align*}
    \sum_{a}p_ay_a &=
                     \sum_{a}p_af_a +
                     \sum_{a}p_a(y_a-f_a).
  \end{align*}
By the AM-GM inequality, we have
  \begin{align*}
    \sum_{a}p_a(y_a-f_a)\leq{}
    \frac{1}{2}\sum_{a}p_a(y_a+f_a) + \frac{1}{2}\sum_{a}p_a\frac{(y_a-f_a)^{2}}{y_a+f_a},
  \end{align*}
so that
  \begin{align*}
    \sum_{a}p_ay_a
    \leq{} \frac{1}{2}\sum_{a}p_ay_a +
    \frac{3}{2}\sum_{a}p_af_a + \frac{1}{2}\sum_{a}p_a\frac{(y_a-f_a)^{2}}{y_a+f_a},
  \end{align*}
  and after rearranging,
  \begin{align*}
    \sum_{a}p_ay_a
    \leq{} 
    3\sum_{a}p_af_a + \sum_{a}p_a\frac{(y_a-f_a)^{2}}{y_a+f_a}.
  \end{align*}
\end{proof}

\subsection{Auxiliary Results}

\begin{proposition}
  \label{prop:exp_concave}
  When $\yhat,y\in\brk*{0,1}$, the logarithmic loss
  $\yhat\mapsto\logl(\yhat,y)$ is $1$-exp-concave and $1$-mixable.
  \end{proposition}

\begin{proof}[\pfref{prop:exp_concave}]
  Let $f_y(\yhat)=\logl(\yhat,y)$. From \citet{hazan2007logarithmic},
  the loss is $\alpha$-exp-concave if and only if
  $f''_y(\yhat)\geq{}\alpha(f'_y(\yhat))^{2}$ for all
  $\yhat,y\in\brk*{0,1}$. We observe that $f'_y(\yhat) = -\frac{y}{\yhat} +
  \frac{1-y}{1-\yhat}$ and $f''_y(\yhat)=\frac{y}{\yhat^{2}} +
  \frac{1-y}{\prn{1-\yhat}^{2}}$. Since $y\in\brk*{0,1}$, Jensen's
  inequality implies that
\[
  (f'_y(\yhat))^{2} \leq y\prn*{\frac{-1}{\yhat}}^2 + (1-y)\prn*{\frac{1}{1-\yhat}}^{2} = f''_y(\yhat),
\]
so we may take $\alpha=1$.

Mixability is an immediate consequence of exp-concavity \citep{PLG}.
\end{proof}

\section{Extensions}
\label{app:extensions}

\subsection{Small Rewards}
\newcommand{\Rstar}{R^{\star}}

In this section we sketch an extension of \mainalg to the setting where the
learner observes rewards $r_t(a) \in [0,1]$ rather than losses
$\ls_t(a)$, and aims to achieve high reward rather than low
loss. As before, we
assume access to a function class $\Fcal$ such that the Bayes predictor $f^\star(x,a) := \EE[r(a) \mid x] \in \Fcal$. Formally, we define regret for this setting as
\[
\RegCB = \sum_{t=1}^{T}r_t(\pistar(x_t)) - \sum_{t=1}^{T}r_t(a_t),
\]
where $\pistar(x)\ldef\argmax_{a\in\cA}\fstar(x,a)$ is the optimal
policy.

Our aim here is to provide regret bounds that adapt whenever the
reward of the optimal policy is small. This type of guarantee is
natural if we believe a-priori that rewards are typically very
small, which is common in personalization and recommendation
applications, where clicks are often used as reward signal, yet
click-through rates are typically well below $1\%$. In such settings, it
is favorable to have regret scaling with the reward $R^\star$ of the
optimal policy. Note that this is \emph{not} equivalent to an $L^\star$
bound after the translation $r_t(a) = 1 - \ell_t(a)$, since having low
reward corresponds to having high loss.

\mainalg can be adapted to the small-reward setting achieve
\[
\En\brk*{\RegCB} \leq \bigoh\prn*{\sqrt{\Rstar\cdot{}A\RegLog} + A\RegLog}
\]
whenever $\En\brk*{\sum_{t=1}^{T}r_t(\pistar(x_t))}\leq{}\Rstar$. The
algorithm remains essentially as described in \pref{alg:online_main}, with
the only difference being that we change the reweighted inverse gap
weighting strategy used in \pref{line:igw}. The new strategy and
corresponding per-round inequality are described in the following theorem.
\begin{theorem}
\label{thm:per_round_rstar}
Let $y \in [0,1]^\K{}$ be given and $b := \argmax_{a}
y_a$. Define $p_a = \frac{y_b}{\K{}y_b + \gamma (y_b-y_a)}$ for $a \ne b$ and
$p_b = 1 - \sum_{a \ne b} p_a$. If $\gamma\geq{}4\K{}$, then for all
$f\in\brk*{0,1}^{\K}$ and $\astar\in\argmax_{a}f_a$, we have
\begin{align*}
      \sum_{a}p_a(f_{\astar}-f_{a}) 
  \leq{} \frac{9\K{}}{\gamma} \sum_a p_a f_a + 10 \gamma \sum_a p_a \frac{(y_a - f_a)^2}{y_a+f_a}.
\end{align*}
\end{theorem}
Observe that the left hand side is the per-round regret of the learner
when $f$ is the \emph{reward} (rather than loss) model, which contrasts with the
\lhs in~\pref{thm:per_round}. On the other hand, the \rhs only differs from that of~\pref{thm:per_round} in the
constants. As such, it naturally yields an $R^\star$ bound when
applied with $y=\yhat_t(x_t,\cdot)$ as in \pref{alg:online_main}.

It should be noted that achieving $R^\star$-based first-order bounds for contextual
bandits appears to be considerably easier than achieving $L^\star$-based bounds. Indeed, the
standard analysis of the \expfour algorithm already yields a
$\bigoh(\sqrt{R^\star\cdot \K\log |\Pi|})$ regret bound, under the
benign assumption that the policy class contains the policy that
selects actions uniformly at random on every context~\citep[][Theorem
  7.1]{auer2002non}. On the other hand, \expfour cannot achieve an $\Lstar$-based bound without modifications \citep{allen2018make}.

\begin{proof}[Proof of~\pref{thm:per_round_rstar}]
The proof parallels that of~\pref{thm:per_round}. We start by adding
and subtracting $y_a$ and applying the AM-GM inequality
\begin{align}
\sum_a p_a(f_{\astar} - f_a) &= \sum_{a\ne \astar} p_a(f_{\astar}-y_a) + \sum_{a \ne \astar} p_a(y_a - f_a)\notag\\
& \leq \sum_{a\ne\astar} p_a(f_{\astar}-y_a) + \frac{1}{4\gamma} \sum_{a\ne\astar} p_a (y_a + f_a) + \gamma\sum_{a\ne\astar}p_a \frac{(y_a - f_a)^2}{y_a+f_a}\notag.
\end{align}
For the first term above, let us consider two cases. 

\paragraph{Case 1} 
First, if $y_b \geq f_{\astar}$ then
\begin{align*}
\sum_{a \ne \astar} p_a (f_{\astar} - y_a) \leq \sum_{a \notin
  \{\astar,b\}} p_a (f_{\astar} - y_a) \leq \sum_{a \notin
  \{\astar,b\}}p_a (f_{\astar} - y_a)\one\{ f_{\astar} \geq y_a\} .
\end{align*}
Here we have simply dropped negative terms. Now, using the definition of
$p_a$ for $a\neq{}b$, we have
\begin{align*}
  p_a (f_{\astar} - y_a)\one\{f_{\astar} \geq y_a\} = \frac{y_b (f_{\astar} - y_a)}{\K{}y_b + \gamma(y_b - y_a)}\one\{f_{\astar} \geq y_a\} \leq \frac{y_b (f_{\astar} - y_a)}{\gamma (y_b - y_a)}\one\{f_{\astar} \geq y_a\}.
\end{align*}
Observe that $y_b/(y_b - y_a) \leq f_{\astar}/(f_{\astar} - y_a)$,
since $y_b \geq f_{\astar} \geq y_a\geq{}0$. This yields
\begin{align*}
\one\{f_{\astar} \geq y_a\} \frac{y_b (f_{\astar} - y_a)}{\gamma (y_b - y_a)} \leq \one\{f_{\astar} \geq y_a\} \frac{ f_{\astar} (f_{\astar}-y_a)}{\gamma (f_{\astar} - y_a)} \leq \frac{f_{\astar}}{\gamma}.
\end{align*}
And so, if $y_b \geq f_{\astar}$ we have the bound
\begin{align*}
\sum_a p_a (f_{\astar} - f_a) \leq \frac{\K{} f_{\astar}}{\gamma} + \frac{1}{4\gamma}\sum_{a \ne \astar} p_a (f_a + y_a) + \gamma\sum_{a \ne \astar} p_a \frac{(y_a - f_a)^2}{y_a + f_a}.
\end{align*}

\paragraph{Case 2}
If $y_b \leq f_{\astar}$ then for the first term, we write
\begin{align}
\sum_{a \ne \astar} p_a (f_{\astar} - y_a) = \sum_{a \notin\{\astar,b\}} p_a(y_b - y_a) + (1-p_{\astar})(f_{\astar} - y_b) \leq \sum_{a\notin\{\astar,b\}}p_a(y_b - y_a) + (f_{\astar}-y_b). \label{eq:rstar_case_2_1}
\end{align}
For the first term in \pref{eq:rstar_case_2_1}, using the definition of $p_a$, we have
\begin{align}
\sum_{a \notin\{\astar,b\}} p_a(y_b - y_a) = \sum_{a \notin\{\astar,b\}} \frac{y_b(y_b - y_a)}{\K{}y_b + \gamma(y_b - y_a)} \leq \sum_{a \notin\{\astar,b\}} \frac{y_b}{\gamma} \leq \frac{\K{}y_b}{\gamma} \leq \frac{\K{}f_{\astar}}{\gamma}. \label{eq:rstar_case_2_2}
\end{align}

For the second term, we first note that $p_b = 1 - \sum_{a\ne b} p_a \geq
1 - \sum_{a \ne b}\frac{y_b}{\K{}y_b} \geq \frac{1}{\K{}}$, then
consider two subcases.
\paragraph{Case 2a ($y_b \leq f_{\astar}$ and $\astar=b$)}
Here we simply use the AM-GM inequality to show that
\begin{align*}
f_{\astar} - y_b = f_{\astar} - y_{\astar} &\leq \frac{f_{\astar} + y_{\astar}}{8\gamma p_{\astar}} + 2\gamma p_{\astar} \frac{(y_{\astar} - f_{\astar})^2}{y_{\astar}+f_{\astar}} \\
& \leq \frac{\K{}}{4\gamma}f_{\astar} + 2\gamma p_{\astar} \frac{(y_{\astar} - f_{\astar})^2}{y_{\astar}+f_{\astar}}.
\end{align*}
Here the first inequality is AM-GM, while the second uses that
$y_{\astar} = y_b\leq f_{\astar}$ (by the conditions for this case), along with the
fact that $p_{\astar} = p_b \geq 1/\K{}$.

\paragraph{Case 2b ($y_b \leq f_{\astar}$ and $\astar \neq b$)}
In this case, we have
\begin{align*}
y_b \geq y_{\astar}, \quad \textrm{and} \quad f_{\astar} \geq f_b.
\end{align*}
Using the definition for $p_{\astar}$, we have
\begin{align*}
f_{\astar} - y_b &= p_{\astar} \frac{\K{}y_b + \gamma(y_b - y_{\astar})}{y_b} (f_{\astar} - y_b) = p_{\astar} \K{} (f_{\astar} - y_b) + p_{\astar}\gamma \frac{(y_b - y_{\astar})(f_{\astar} - y_b)}{y_b}\\
& \leq p_{\astar} \K{} (f_{\astar} - y_b) + p_{\astar}\gamma \frac{(f_{\astar} - y_{\astar})(f_{\astar} - y_b)}{f_{\astar}}\\
& = p_{\astar} \K{} (f_{\astar} - y_b) + p_{\astar}\gamma \frac{(f_{\astar} - y_{\astar})^2}{f_{\astar}} + p_{\astar}\gamma \frac{(f_{\astar} - y_{\astar})(y_{\astar} - y_b)}{f_{\astar}}\\
& \leq p_{\astar} \K{} (f_{\astar} - y_b) + p_{\astar}\gamma \frac{(f_{\astar} - y_{\astar})^2}{f_{\astar}}\\
& \leq p_{\astar} \K{} (f_{\astar} - y_{\astar}) + p_{\astar}\gamma \frac{(f_{\astar} - y_{\astar})^2}{f_{\astar}}.
\end{align*}
Here, in the first inequality we use that $a \mapsto (a-b)/a$ is
increasing in $a$, for $a,b \geq 0$ along with the fact that
$f_{\astar} \geq y_b \geq y_{\astar}$. The second and third
inequalities both use that $y_{\astar} \leq y_b$.

Now by the AM-GM inequality, we have
\begin{align*}
p_{\astar} \K{} (f_{\astar} - y_{\astar}) &\leq \frac{p_{\astar} \K{}^2}{4\gamma} f_{\astar} + p_{\astar} \gamma \frac{(f_{\astar} - y_{\astar})^2}{f_{\astar}}\\
& \leq \frac{\K{}}{4\gamma}f_{\astar} + p_{\astar} \gamma \frac{(f_{\astar} - y_{\astar})^2}{f_{\astar}},
\end{align*}
where the second inequality uses the fact that $p_{\astar} \leq 1/\K{}$ since $a_\star \ne b$. 
Finally, we use that $y_{\astar} \leq f_{\astar}$ to conclude that in
this case,
\begin{align}
f_{\astar} - y_b \leq \frac{\K{}}{4\gamma} f_{\astar} + 4\gamma p_{\astar} \frac{(f_{\astar} - y_{\astar})^2}{f_{\astar} + y_{\astar}}. \label{eq:rstar_subcases}
\end{align}
This bound applies to both Case 2a and 2b. 

\paragraph{Wrapping up}
Returning to Case 2 and combining~\pref{eq:rstar_case_2_1},~\pref{eq:rstar_case_2_2}, and~\pref{eq:rstar_subcases}, we have
\begin{align*}
\sum_{a \ne \astar} p_a (f_{\astar} - y_a) \leq \frac{2\K{}f_{\astar}}{\gamma} + 4\gamma p_{\astar} \frac{(f_{\astar} - y_{\astar})^2}{f_{\astar} + y_{\astar}}.
\end{align*}
Combining this with our initial calculation, we have
\begin{align*}
\sum_{a} p_a (f_{\astar} - f_a) &\leq \frac{2\K{}f_{\astar}}{\gamma} + 4\gamma p_{\astar} \frac{(f_{\astar} - y_{\astar})^2}{f_{\astar} + y_{\astar}}  + \frac{1}{4\gamma}\sum_{a \ne \astar} p_a (y_a + f_a) + \gamma \sum_{a \ne \astar} p_a \frac{(y_a - f_a)^2}{y_a + f_a}\\
& \leq 4 \gamma \sum_a p_a \frac{(y_a - f_a)^2}{y_a + f_a} + \frac{1}{4\gamma}\sum_{a} p_a (y_a + f_a) + \frac{2\K{}}{\gamma} f_{\astar}.
\end{align*}
Next, we can apply~\pref{lem:y_to_f} as-is, which yields
\begin{align*}
\sum_{a} p_a (f_{\astar} - f_a) \leq \frac{1}{\gamma} \sum_a p_a f_a + 5\gamma \sum_a p_a\frac{(y_a - f_a)^2}{y_a+f_a} + \frac{2\K{}}{\gamma} f_{\astar}.
\end{align*}
This inequality, after using assumption the that $\gamma \geq 4\K{}$
and rearranging, implies
\begin{align*}
f_{\astar} \leq 2 (1+1/\gamma)\sum_a p_a f_a + 10\gamma \sum_a p_a \frac{(y_a - f_a)^2}{y_a+f_a} \leq 4\sum_a p_a f_a + 10 \gamma \sum_a p_a \frac{(y_a - f_a)^2}{y_a+f_a}.
\end{align*}
Plugging this bound in for the final expression gives
\begin{align*}
\sum_{a} p_a (f_{\astar} - f_a) &\leq \frac{1}{\gamma} \sum_a p_a f_a + 5\gamma \sum_a p_a\frac{(y_a - f_a)^2}{y_a+f_a} + \frac{8\K{}}{\gamma} \sum_a p_a f_A + 20 \K{} \sum_a p_a \frac{(y_a - f_a)^2}{y_a+f_a}\\
& \leq \frac{9\K{}}{\gamma} \sum_a p_a f_a + 10 \gamma \sum_a p_a \frac{(y_a - f_a)^2}{y_a+f_a},
\end{align*}
as desired.
\end{proof}

\section{Details for Experiments}
\label{app:experiments}

\subsection{Assets and Computing Resources}

\paragraph{Assets} The code for the contextual bandit evaluation setup of \citet{bietti2018contextual}, which we used as a starting point, is
publicly available at
\url{https://github.com/albietz/cb_bakeoff}. Likewise, the source code
for Vowpal Wabbit, upon which our implementation is built, is publicly available at
\url{https://github.com/vowpalwabbit/vowpal_wabbit/}. \neurips{The source code
used to run the experiments is included in the supplementary material.}

All datasets used in the experiments are publicly available via the
OpenML collection (\url{https://www.openml.org}). Readers
can refer to
the information page for each respective dataset (e.g.,
\url{https://www.openml.org/d/1041} for dataset \texttt{1041}) for
copyright information.

\paragraph{Computing resources} Experiments were run on a single \texttt{n1-highcpu-32} instance on
Google Compute Engine. The total compute time required to run the experiments was
under 12 hours.

\subsection{Additional Details}

\paragraph{Datasets}
We restrict to a subset of the bake-off suite consisting of 516
multiclass classification datasets in the same fashion as
\citet{foster2020instance}.

\paragraph{Algorithms and oracle}
For \squarecbl and \fastcbl we take $\cF$ to be a class of generalized linear
models:
\begin{equation}
  \label{eq:glm}
  \cF=\crl*{(x,a)\mapsto\sigma(\tri{w,\phi(x,a)}) \mid{} w\in\bbR^{d}},
\end{equation}
where $\sigma(t)=1/(1+e^{-t})$ is the logistic link function and
$\phi(x,a)$ is a fixed (dataset-dependent) feature map. 
This choice is
convenient because i) it naturally produces predictions in
$\brk*{0,1}$, as required by \mainalg, and ii), we have that
$\logloss(\sigma(\tri{w,\phi(x,a)}),y)=\ls_{\mathrm{logistic}}(\tri{w,\phi(x,a)},
y)$, so that online regression with the logarithmic loss is equivalent
to online logistic regression (cf. \pref{ex:logistic}).

Even though
\squarecb is designed for the square loss rather than the \loglosst, one
can show that under the realizability assumption
(\pref{ass:realizability}), any \loglosst oracle is admissible for
\squarecb. Indeed, for any \loglosst oracle satisfying
\pref{ass:logloss_regret}, realizability and Pinsker's inequality
imply that
\begin{equation}
  \label{eq:kl_square}
  \En\brk*{
    \sum_{t=1}^{T}\prn*{\yhat_t(x_t,a_t)-\fstar(x_t,a_t)}^{2}
  }
  \leq{} 2   \En\brk*{
    \sum_{t=1}^{T}\dkl{\fstar(x_t,a_t)}{\yhat_t(x_t,a_t)}
    } \leq{} 2\RegLog,
  \end{equation}
  which means that the oracle is a valid square loss oracle for
  \squarecb in
  the sense of Assumption 2b in \citet{foster2020beyond}.

   The oracle is trained with the default VW learning
  rule, which performs online gradient descent with adaptive
  updates~\citep{duchi2011adaptive,karampatziakis2011online,ross2013normalized}. We
  treat the algorithm's step size parameter as a tunable hyperparameter.

For \squarecbs, we configure \squarecb exactly as described in
\citet{foster2020instance}. We take $\cF$ to be the class of linear
models
\[
  \cF=\crl*{(x,a)\mapsto\tri{w,\phi(x,a)} \mid{} w\in\bbR^{d}},
\]
and the oracle applies the default VW learning rule to the
square loss. We use the same hyperparameter range as for \squarecbl
and \fastcbl, both for the \squarecb learning rate and for the VW
learning rule's step size.

\paragraph{Tables in \pref{fig:experiments}}
For both tables, each cell $(a,b)$ plots the number of datasets in
which algorithm $a$ significantly beats $b$, minus the number of
datasets in which $b$ significantly beats $a$. Following
\citet{bietti2018contextual}, we define a significant win using a
heuristic based on an approximate $Z$-test. If $p_a$ and $p_b$ are the
final PV loss values for algorithms $a$ and $b$, respectively, we say that $a$
significantly beats $b$ if 

\begin{equation}
  \label{eq:ztest}
  1 - \Phi\prn*{
    \frac{p_a-p_b}{\sqrt{\frac{p_a(1-p_a)}{n}+\frac{p_b(1-p_b)}{n}}
    }} < 0.05,
  \end{equation}
  where $n$ is the number of examples and $\Phi$ is the Gauss error function.

In the left table, we choose the configuration
(hyperparameters for \squarecb/\mainalg and learning rate for the VW
learner) with lowest final PV loss for each algorithm on a per-dataset
basis. In the right table, for each algorithm we choose the hyperparameter configuration with best
performance on a held-out collection of 200 datasets using the method
described in \citet{bietti2018contextual}. We keep this configuration
fixed and tune only the learning rate for the VW learner on each dataset.

\paragraph{Plots in \pref{fig:experiments}}
Each plot shows the progressive validation loss $L_{\textsf{PV}}(t)$ as a
function of the number of examples $t$, for the best-performing (in
terms of final PV loss) hyperparameter configuration for each
algorithm. We consider 10 replicates for each dataset, where
  each replicate has the example order randomly permuted, and plot the
average progressive validation loss across the replicates. Error bands
in each plot correspond to significance $p < 0.05$ under the $Z$-test
in \pref{eq:ztest}, setting $n=t\cdot(\mathrm{\# replicates})$ at each time $t$.

The algorithm \textsf{Supervised.L} included in each of the plots is
an oracle benchmark that runs online logistic regression using the true
label for each example (which the bandit algorithms do not have access
to). The only hyperparameter for this algorithm is the learning rate
for the VW learning rule.

\subsection{Additional Figures}

\pref{fig:ada} shows the results for the experiment in
\pref{fig:experiments} (Top-Left) with two additional adaptive algorithms, \adacb and
\regcb, included. These algorithms were found to have the strongest
overall performance on the bake-off suite in
\citet{foster2020instance} using the same online square loss oracle as
\squarecbs, and are considered state-of-the-art
\citep{bietti2018contextual,foster2020instance}. We see in that switching
\squarecb from regression with the square loss to the logistic
loss (\squarecbl) is already enough to outperform \adacb and \regcb, and
that the performance of \fastcbl is even stronger. It would be
interesting to understand how the performance of \adacb and \regcb
improves if we switch to the generalized linear model \pref{eq:glm} in
the same fashion as \squarecbl/\fastcbl, but it is unclear how to
efficiently compute the confidence sets required by these algorithms in
this case. We leave this for future work.

\begin{figure}[h!]
\centering
    \begin{tabular}{ | l | c | c | c | c | c | }
      \hline
      $\downarrow$ vs $\rightarrow$ & R.S & A.S & S.S & S.L & F.L \\ \hline
      RegCB.S & - & 6 & 46 & -6 & -12 \\ \hline
      AdaCB.S & -6 & - & 42 & -8 & -18 \\ \hline
      SquareCB.S & -46 & -42 & - & -55 & -66 \\ \hline
      SquareCB.L & 6 & 8 & 55 & - & -11 \\ \hline
      \textbf{FastCB.L} & \textbf{12} & \textbf{18} & \textbf{66} & \textbf{11} & - \\ \hline
    \end{tabular}
  \caption{Head-to-head win-loss differences. Each entry indicates the statistically
    significant win-loss difference between the row algorithm and the column
    algorithm. Hyperparameters are per-dataset.}
\label{fig:ada}
\end{figure}

\end{document}